%% file: main.tex
\documentclass{article}

\usepackage{microtype}
\usepackage{graphicx}
\usepackage{caption}
\usepackage{subcaption}
\usepackage{booktabs} % for professional tables
\usepackage{tabularx}
\usepackage{booktabs}
\usepackage{multicol}
\usepackage{hyperref}

\usepackage[accepted]{icml2025}
\usepackage{algorithm}
\usepackage{algpseudocode}

\usepackage{amsmath}
\usepackage{amssymb}
\usepackage{mathtools}
\usepackage{amsthm}
\usepackage{amsthm}
\usepackage{amsfonts}
\input{macros.tex}

\usepackage[textsize=tiny]{todonotes}
\usepackage[nameinlink, noabbrev, capitalise]{cleveref}
\crefname{equation}{}{}
\crefname{figure}{Fig.}{Figs.}
\crefname{section}{Sec.}{Secs.}
\crefname{appendix}{App.}{Apps.}
\crefname{table}{Tab.}{Tabs.}

\newtheorem{theorem}{Theorem}
\crefname{theorem}{Thm.}{Thms.}
\newtheorem{lemma}{Lemma}
\crefname{lemma}{Lem.}{Lems.}

\crefname{assumption}{Assump.}{Assumps.}
\newtheorem{proposition}{Proposition}
\crefname{proposition}{Prop.}{Props.}

\crefname{corollary}{Cor.}{Cors.}

\crefname{definition}{Def.}{Defs.}

\crefname{remark}{Rmk.}{Rmks.}
\theoremstyle{remark}
\crefname{algocf}{Alg.}{Algs.}
\crefname{algorithm}{Alg.}{Algs.}

\icmltitlerunning{Multimodal Diffusion on Arbitrary State Spaces} 
\begin{document}

\twocolumn[
\icmltitle{Diffuse Everything: Multimodal Diffusion Models on Arbitrary State Spaces}

\icmlsetsymbol{equal}{*}
\begin{icmlauthorlist}
\icmlauthor{Kevin Rojas}{equal,gtmath,gtml}
\icmlauthor{Yuchen Zhu}{equal,gtmath,gtml}
\icmlauthor{Sichen Zhu}{gtml}
\icmlauthor{Felix X.-F. Ye}{sunymath}
\icmlauthor{Molei Tao}{gtmath,gtml}
\end{icmlauthorlist}

\icmlaffiliation{gtml}{Machine Learning Center, Georgia Institute of Technology, Atlanta, GA}
\icmlaffiliation{gtmath}{School of Mathematics, Georgia Institute of Technology, Atlanta, GA}
\icmlaffiliation{sunymath}{Department of Mathematics \& Statistics, SUNY Albany, NY}

\icmlcorrespondingauthor{Molei Tao}{\texttt{mtao@gatech.edu}}

\icmlkeywords{Multimodal Diffusion Models}

\vskip 0.3in
]

\printAffiliationsAndNotice{\icmlEqualContribution} % otherwise use the standard text.

\begin{abstract}

Diffusion models have demonstrated remarkable performance in generating unimodal data across various tasks, including image, video, and text generation. On the contrary, the joint generation of multimodal data through diffusion models is still in the early stages of exploration. Existing approaches heavily rely on external preprocessing protocols, such as tokenizers and variational autoencoders, to harmonize varied data representations into a unified, unimodal format. This process heavily demands the high accuracy of encoders and decoders, which can be problematic for applications with limited data. To lift this restriction, we propose a novel framework for building multimodal diffusion models on arbitrary state spaces, enabling native generation of coupled data across different modalities. By introducing an innovative decoupled noise schedule for each modality, we enable both unconditional and modality-conditioned generation within a single model simultaneously. We empirically validate our approach for text-image generation and mixed-type tabular data synthesis, demonstrating that it achieves competitive performance.
\end{abstract}

\section{Introduction}
Recent years have witnessed the tremendous success of diffusion generative models in various applications. The seminal works of continuous diffusion models on Euclidean spaces \citep{sohl2015deep, song2020score, ho2020denoising} have led to state-of-the-art methods for tasks such as image generation \citep{dhariwal2021diffusion, bao2023all, karras2022elucidating, karras2024analyzing}, video generation \citep{ho2022video, jin2024pyramidal}, time series forecasting \citep{chen2024probabilistic,rojas2025variational} and in domains such as robotics \citep{chi2023diffusion} and genomics \citep{luo2024scdiffusion, zhu2025diffusion}. Pioneering works have shown that diffusion models can also be extended to curved spaces \citep{de2022riemannian, huang2022riemannian, cheng2024stiefel, chen2024flow, zhu2024trivialized}, enabling a high-fidelity generation of structured data on manifolds, such as material configurations \citep{sriram2024flowllm} and protein backbones \citep{watson2023novo, yim2023se}. Recently, discrete diffusion models have emerged as the cornerstone for modeling categorical data with inherent discrete structures \citep{campbell2022continuous, loudiscrete, campbell2024generative, gat2024discrete}. Discrete diffusion models have imposed great impacts on protein sciences \citep{wang2024diffusion}, graph generation \citep{xu2024discrete, li2024layerdag}, and text generation \citep{nie2024scaling, nie2025large}. In general, diffusion models have shown top performance in most scenarios with unimodal data. 

Generative models also demonstrated successes in multimodal data lately. For example, conditional diffusion models showed remarkable capabilities in tasks such as text-to-image generation by accurately synthesizing pictures following given instruction prompts \citep{ramesh2022hierarchical, chen2023pixart, esser2024scaling}. It's worth noting that such models still generate single-modality outputs (such as images). Therefore, to jointly generate multimodal data, leveraging only single-task-performing conditional models is extremely computationally inefficient, as it requires combining \textbf{multiple} independently trained models by sequentially applying them.

\begin{figure*}[!t]
\vspace{-7pt}
    \centering
    \includegraphics[width=0.9\textwidth]{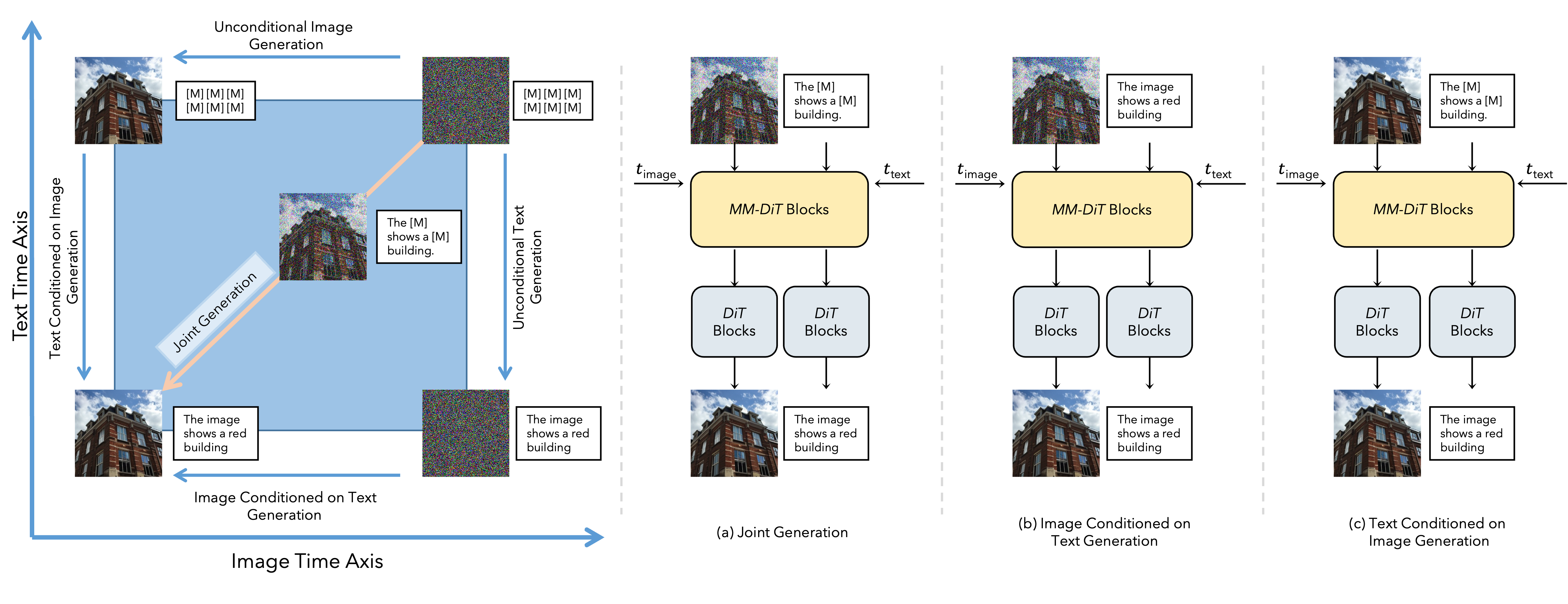}
    \vspace{-2em}
\caption{By injecting noise into different modalities in a decoupled fashion, we enable the unconditional and modality-conditioned generation in a single model. (a) Joint generation of image and text. (b) Image generation given text captions as conditions. (c) Text generation given images as conditions. }
\label{fig:promo}
\vspace{-1em}
\end{figure*}

An alternative approach is to use a \textbf{single} multi-modal model that captures the joint distribution of multiple modalities. Such an approach often leads to strong performances as it allows information to mix across modalities \citep{li2024omniflow, team2024chameleon}. Existing approaches of this type are mainly based on autoregressive models (AR), such as Chameleon \citep{team2024chameleon} and Unified-IO \citep{lu2024unified}, where data of different modalities are represented uniformly as tokens and generated autoregressively from left to right. Apart from these, attempts have also been made to realize this idea using diffusion/flow-based methods, such as UniDiffuser \citep{bao2023one}, MM-Diffusion \citep{ruan2023mm}, AVDiT \citep{kim2024versatile}, UniDisc \citep{swerdlow2025unified}, OmniFlow \citep{li2024omniflow}, etc. These methods generate multimodal data simultaneously through iterative denoising of a randomly sampled initial noise.

A commonality among the aforementioned methods designed for joint multimodal generation is that they typically rely heavily on preprocessing techniques to \textbf{harmonize the varied data representations into a unified format}, thereby creating a single modality. One approach (taken by Chameleon \citep{team2024chameleon}, UniDisc \citep{swerdlow2025unified}, etc) is to cast multimodal inputs all into discrete tokens with modality-dependent tokenizers built with discrete or vector-quantized variational encoders (VQVAE) \citep{van2017neural, esser2021taming}. An alternative route (considered by UniDiffuser \citep{bao2023one}, OmniFlow \citep{li2024omniflow}, etc) is to preprocess the multimodal data by embedding them into continuous-valued latent vectors with encoders trained with variational encoders (VAE) \citep{kingma2013auto} or representation alignment (e.g., CLIP \citep{radford2021learning}). 

For these approaches, regardless of whether discrete tokens or continuous latents are used, generation is performed in a unimodal space, and the original data modality must be recovered through decoding. Therefore, these pipelines may suffer from generation artifacts due to the limited accuracy of the decoders \citep{hoogeboom2024simpler}. Additionally, the requirement for high-performance encoder-decoder pairs can be problematic to satisfy for applications that lack abundant high-quality data \citep{zhang2023mixed}. Due to the requirement for task-specific algorithm designs, these methods also cannot be conveniently extended to generate data composed of arbitrary modalities. Therefore, a natural question to ask is the following:
\begin{center}
\textit{Can we design a principled framework to enable joint modeling of multi-modal data in their native spaces without a unified representation?}
\end{center}
Diffusion models serve as a powerful backbone for building such a framework. Theorists have shown that diffusion models can be extended to a more general idea called \emph{denoising Markov models} \citep{benton2022denoising, ren2025unified}, providing a solid theoretical foundation for a \textbf{multimodal extension}. In addition to this, existing works such as MultiFlow \citep{campbell2024generative} and Generator Matching \citep{holderrieth2024generator} have verified the effectiveness of \textbf{multimodal models in native state spaces} in protein design. Motivated by these successes, we propose a general framework for building multimodal diffusion models on arbitrary state spaces without the need for data format unifiers. Our contributions are three-fold: 
\vspace{-5pt}
\begin{enumerate}
\item We propose a novel framework for building multimodal diffusion models by combining the native diffusion models designed for each data modality, and derive a unified learning objective. Under our design, learning multimodal diffusion models is as straightforward as performing a joint optimization on a sum of unimodal learning losses, despite requiring a non-trivial proof.

\item We introduce decoupled noise schedules for each data modality and theoretically justify the validity of score learning under the presence of multiple time variables. We demonstrate that this design enables us to simultaneously handle both unconditional and conditional multimodal generation in one single model. We also propose a novel guidance mechanism effective in both use cases for enhancing generation quality.

\item We experiment with text-image generation and mixed-type tabular data synthesis, achieving competitive performance on both tasks with more parameter-efficient models, without relying on pre-trained models or powerful extra encoders. More importantly, we devise a set of training strategies for the task of text-image, which is crucial for achieving success.
\end{enumerate}

\begin{figure*}[!t]
\begin{center}
    \includegraphics[width=0.75\textwidth]{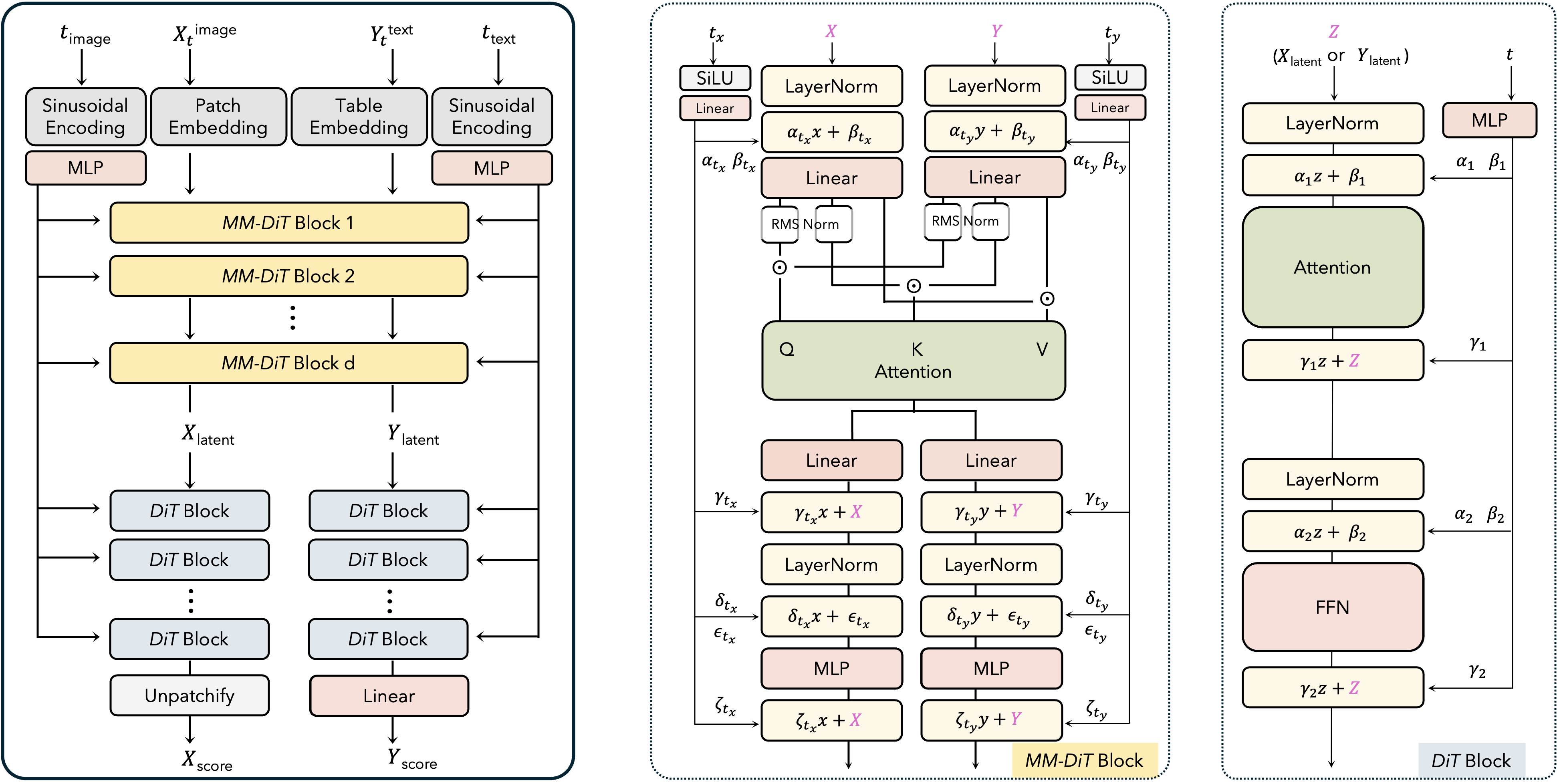}
\centering
\end{center}
\vspace{-1em}
\caption{Network backbone for text-image generation, motivated by MMDiT \citep{esser2024scaling} and DiT \citep{peebles2023scalable}.}
\vspace{-1em}
\label{fig:architecture}
\end{figure*}

\section{Preliminaries}
In this section, we review basic concepts and formulations of common diffusion models in different state spaces.

\subsection{Continuous Space Diffusion Models}
For continuous diffusion models \citep{song2020score, ho2020denoising}, one considers a continuous time stochastic differential equation (SDE) $\{X_t\}_{0 \leq t \leq T}$ on the Euclidean space $\mathbb{R}^{d}$ as the \emph{forward process}. The process is characterized by the following dynamic, 
\begin{align*}
    \rd X_t = \boldsymbol{f}(X_t, t) \rd t + \boldsymbol{g}(t) \rd W_t,
\end{align*}
where $\boldsymbol{f}: \R^d \times \R \to \R^d$ is the drift and $\boldsymbol{g} : \R \to \R$ is the diffusion coefficient. We denote $p_t = \Law(X_t)$. $\boldsymbol{f}$ and $\boldsymbol{g}$ are often chosen that $p_{T}$ is an easy-to-sample distribution. A popular pick is the time re-parametrized Ornstein Uhlenbeck process, which corresponds to the selection of $\boldsymbol{f}(X_t, t) = - \frac{1}{2}\beta_t X_t$ and $\boldsymbol{g}(t) = \sqrt{\beta_t}$, for some positive noise schedule $\beta_t$. In such case, $p_{T} \approx \mathcal{N}(0, I)$ for reasonably large $T$. It can be shown that the backward process is another SDE with a different drift \citep{anderson1982reverse},
\begin{align*}
    \rd X_t = \boldsymbol{f}(X_t, t) \rd t - \boldsymbol{g}^2(t) \nabla_{x} \log p_{t}(X_t) \rd t + \boldsymbol{g}(t) \rd W_t
\end{align*}
The common practice in training is to define the score vector $\boldsymbol{s}(X_t, t) = \nabla_{x} \log p_{t}(X_t)$ and we approximate it with a neural network $\boldsymbol{s}_{\theta}$, estimated by minimizing a variant of the following score matching loss \citep{vincent2011connection},
\begin{align}
\label{eq:continous_sm}
\min_\theta \int_{0}^{T} \mathbb{E}_{X_t \sim p_t} \Big[ \big\|\boldsymbol{s}_{\theta}(X_t, t) - \boldsymbol{s}(X_t, t) \big\|^2 \Big] \rd t.
\end{align}

\subsection{Discrete Space Diffusion Models}
For discrete diffusion models \citep{campbell2022continuous, loudiscrete, ou2024your, sahoo2024simple, shi2024simplified}, one considers a continuous time markov chain (CTMC) $\{X_t\}_{0 \leq t \leq T}$ on a finite state space $\mathbb{X}$ as the \emph{forward process}. The distribution of $X_t$ is represented by a vector $p_t$ in the probability simplex on $\mathbb{R}^{|\mathbb{X}|}$. The dynamic of $X_t$ can be characterized by the following equation,
\begin{align*}
    \frac{\rd p_t}{\rd t} = \boldsymbol{Q}_t p_t, \, \text{where } \boldsymbol{Q}_t = (Q_t(x, y))_{x, y \in \mathbb{X}}
\end{align*} 
is a transition matrix satisfying that for any $x \in \mathbb{X}$, $Q_t(x, x) = - \sum_{y \neq x} Q_t(y, x)$, and for any $x \neq y \in \mathbb{X}$, $Q_t(x, y) \geq 0$. We will also denote the dynamic of $X_t$ using the following notation,
\begin{align*}
    X_{t} \sim \operatorname{CTMC}(\boldsymbol{Q}_t)
\end{align*}
$\boldsymbol{Q}_t$ is often chosen such that $p_{T}$ is a simple distribution, such as uniform on $\mathbb{X}$ or Dirac on a masked state. Common choices include uniform or masked transition matrix \citep{loudiscrete}. It is known that the backward process is another process of the same form but with a different transition rate matrix \citep{kelly2011reversibility}, which can be described as
\begin{align*}
    X_{t} \sim \operatorname{CTMC}(\overline{\boldsymbol{Q}}_t)
\end{align*}
where the rate matrix $\overline{\boldsymbol{Q}}_t = (\overline{Q}_t(x,y))_{x, y \in \mathbb{X}}$ is defined as,
\begin{equation*}
    \overline Q_t(y, x) = \begin{cases}
        \tfrac{p_{T-t}(y)}{p_{T-t}(x)} Q_{T-t}(x, y),\ & x \neq y \in \mathbb{X},\\
        - \sum_{y' \neq x} \overline Q_t(y', x),\ & x = y \in \mathbb{X}.
    \end{cases}
\end{equation*}

In discrete diffusion model training, one usually defines the concrete score vector $\boldsymbol{s}(X_t, t) = \big(\frac{p_t(y)}{p_t(X_t)}\big)_{y \in \mathbb{X}}$, and we approximate it with a neural network $\boldsymbol{s}_{\theta}(X_t, t)$, estimated by minimizing the following a variant of the following score entropy loss \citep{benton2022denoising, loudiscrete}, 
\begin{align}
\label{eq:discrete_sm}
 & \min_\theta \int_0^T \E_{X_t \sim p_t} \bigg[ \sum_{y \neq X_t}Q_t(X_t, y) \Big( \boldsymbol{s}_t(X_t, t)_{y} \log \tfrac{\boldsymbol{s}(X_t, t)_{y}}{\boldsymbol{s}_{\theta}(X_t, t)_y} \notag \\
 & \qquad \qquad - \boldsymbol{s}(X_t, t)_{y} +  \boldsymbol{s}_{\theta}(X_t, t)_{y}\Big)\bigg] \rd t. 
\end{align}

\begin{figure*}[!t]
     \centering
     \begin{subfigure}[t]{0.3\textwidth}
         \centering
         \includegraphics[width=0.67\textwidth]{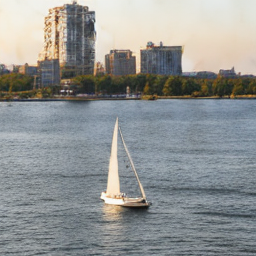}
         \caption{\textbf{Image conditioned on text.} Sample generated using the caption: \textit{``The image features a sailboat sailing on a large body of water, with a city skyline in the background.'' }}
         \label{fig:t2img}
     \end{subfigure}
     \hfill
     \begin{subfigure}[t]{0.3\textwidth}
         \centering
         \includegraphics[width=0.67\textwidth]{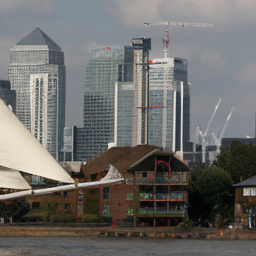}
         \caption{\textbf{Text conditioned on Image.} We generate the following caption: \textit{``The image features a cityscape with a large building, a bridge, and a city skyline in the background. The city is situated near the water. ''}}
         \label{fig:img2text}
     \end{subfigure}
     \hfill
     \begin{subfigure}[t]{0.3\textwidth}
         \centering
         \includegraphics[width=0.67\textwidth]{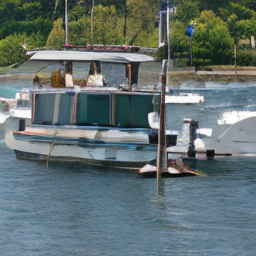}
         \caption{\textbf{Joint generation of text and image.} The caption corresponding to this image was: \textit{``The image features a large white boat with a blue roof, which allows people to look out. The boat is traveling through a body of water.''}}
         \label{fig:five over x}
     \end{subfigure}
        \caption{Visualization of samples generated by our approach. Captions are truncated for brevity. }
        \vspace{-1em}
        \label{fig:uncond}
    \label{fig:qualitative-samples}
\end{figure*}

\subsection{Riemannian Diffusion Models}
For Riemannian diffusion models, e.g. \citep{de2022riemannian, huang2022riemannian}, one considers a continuous time SDE $\{X_t\}_{0 \leq t \leq T}$ on the manifold $\mathcal{M}$ as the \emph{forward process}, characterized by the following dynamic, 
\begin{align*}
    \rd X_t = \boldsymbol{f}(X_t, t) \rd t + \rd W^{\mathcal{M}}_t,
\end{align*}
where $\boldsymbol{f}: \mathcal{M} \times \R \to T_{x}\mathcal{M}$ is the drift, and $\rd W^{\mathcal{M}}_t$ is the manifold Brownian motion. For a compact manifold $\mathcal{M}$, one can pick $\boldsymbol{f} = 0$, and the corresponding $p_{T} \approx \operatorname{Uniform}(\mathcal{M})$ is the uniform distribution on the manifold for large $T$. It's proved that the backward process is another SDE on the manifold \citep{de2022riemannian},
\begin{align*}
    \rd X_t = \boldsymbol{f}(X_t, t) \rd t - \nabla \log p_t(X_t) \rd t + \rd W^{\mathcal{M}}_t,
\end{align*}
where $\nabla$ is the Riemannian gradient on $\mathcal{M}$. For Riemannian diffusion model training, similar to the continuous case, one defines the score $\boldsymbol{s}(X_t, t) = \nabla \log p_{t}(X_t)$ and approximates it with a neural network $\boldsymbol{s}_{\theta}$, which typically requires special design to meet the requirement $\boldsymbol{s}_{\theta}(X_t, t) \in T_{X_t}\mathcal{M}$. Learning is performed through a variant of Riemannian score matching \citep{de2022riemannian, huang2022riemannian},
\begin{align}
\label{eq:riem_sm}
\min_\theta \int_{0}^{T} \mathbb{E}_{X_t \sim p_t} \Big[ \big\|\boldsymbol{s}_{\theta}(X_t, t) - \boldsymbol{s}(X_t, t) \big\|_{\mathcal{M}}^2 \Big] \rd t.
\end{align}

\section{Methodology}
In this section, we present the framework for constructing multimodal diffusion models on general state spaces in their native forms. We first discuss a unified perspective on diffusion models and then present a learning algorithm for diffusion generative modeling of multiple data modalities, where each modality has an independent time variable. Such a framework allows \textbf{any-to-any modality generation} by one single model, which simultaneously includes the \textbf{joint unconditional generation of all modalities} and the \textbf{conditional generation of a subset of modalities given the rest}. Finally, we discuss the Continuous-Discrete Multimodal Diffusion as an application of this framework to illustrate its importance and flexibility.

\subsection{Unified Perspective on Unimodal Diffusion Models}
While common unimodal diffusion models (continuous, discrete, Riemannian, etc) have distinct forward/backward processes, learning objectives, and score parameterizations, they are essentially the realization of \emph{denoising Markov models} \citep{benton2022denoising, ren2025unified} in different situations. At a high level, diffusion models consist of a Markovian forward process that gradually injects `noise' to transform the target data distribution into a simple distribution, and a backward generative process that inverts it using information learned from the marginals of the forward process. This abstract view is formally summarized below.

Consider a Markov process $\{X_{t}\}_{0 \leq t \leq T}$ with $X_{0} \sim p_{\text{data}}$, defined on a state space $\mathcal{X}$. A Markov process can be conveniently characterized using the notion of infinitesimal generators. Since $X_t$ is not necessarily time-homogeneous, we instead consider the augmented process $\bar{X} = (X_t, t)$ defined on the augmented space $\mathcal{X}_{*} = \mathcal{X} \times [0, +\infty)$. Under mild regularity assumptions, $\bar X$ is a Feller process, and its generator $\mathcal{L}$ can be defined as,
\begin{align*}
\mathcal{L} f(x) = \lim_{t \to 0} \frac{\mathbb{E}[f(X_t) | X_0 = x] - f(x)}{t},
\end{align*}
and $f:\mathcal{X}_{*} \to \mathbb{R}$ is a class of test function. We can understand the generator through its decomposition $\mathcal{L} = \partial_t + \hat{\mathcal{L}}$, where $\hat{\mathcal{L}}$ is an operator that acts on functions defined on the original space $\mathcal{X}$. We can also unify the characterization of the evolution of marginals $p(\cdot, t) = \Law(X_t)$ as well as the score learning objectives in terms of $\mathcal{L}$. Under weak technical assumptions, $p(\cdot, t)$ satisfies the following general form of Fokker-Planck equation,
\begin{align*}
    \partial_t p(x, t) = \hat{\mathcal{L}}^{*}p(x, t), \, p(x, 0) = p_{\text{data}}(x).
\end{align*}
where $\hat{\mathcal{L}}^*$ is the adjoint operator of $\hat{\mathcal{L}}$. Moreover, we can define a \emph{generalized explicit score matching} objective \citep{lyu2012interpretation, benton2022denoising},
\begin{align}
\mathcal{J}_{\text{ESM}}(\beta) = \mathop{\mathbb{E}}_{t, p_{t}} \Big[ \Phi\left(\frac{p}{\beta}\right)(X_t, t)\Big]
\end{align}
where $\Phi(f) = f^{-1} \mathcal{L} f - \mathcal{L} \log f$, and $\beta: \mathcal{X} \times [0, +\infty) \to \mathbb{R}^{+}$ . Note that this is an extension of the common score matching objectives, and we could interpret $\mathcal{J}_{\text{ESM}}$ as a loss function that compares the `gradient log' of $p(x, t)$ to that of $\beta(x,t)$, through which we learn the information of the forward marginal $p_{t}$. For example, in continuous diffusion with $\boldsymbol{g}(t) = 1$, $\mathcal{L} = \partial_t + \boldsymbol{f} \cdot \nabla + \frac{1}{2} \Delta$, we recover the explicit score matching objective on Euclidean space \citep{hyvarinen2005estimation}, with $\Phi(p/\beta) = \frac{1}{2} \| \nabla \log p - \nabla \log \beta \|^2$. Note that the explicit score matching objectives are not tractable for training purposes; we will later introduce their equivalent, trainable variants.

\subsection{Versatile Multimodal Diffusion Models with Decoupled Times}
Building on this unified description of unimodal diffusion models, we extend the denoising Markov model framework to a multimodal scenario. This enables us to perform generative modeling of data distributions consisting of mixed-type data without requiring complicated preprocessing pipelines. 

We begin by formally defining the \emph{forward process} in terms of generators of Markov processes. Assume that we have a data distribution $p_{\text{data}}$ defined on the product state spaces $\mathcal{X}^{1} \times \dots \times \mathcal{X}^{n}$. For $1 \leq i \leq n$, we have a Markov process $\{X^{i}_{t^{i}}\}_{1 \leq t^{i} \leq T}$ on $\mathcal{X}^i$, which is regular enough with a unique, easy-to-sample stationary distribution $\pi^{i}$. We pick $T$ so that $\Law(X^{i}_{T}) \approx \pi^i$ for each $1 \leq i \leq n$. Now we can introduce the following joint forward process,
\begin{align}
\label{eq:joint-forward}
& \boldsymbol{X}_{\boldsymbol{t}} = (X^1_{t^1}, \dots, X^i_{t^i}, \dots, X^n_{t^{n}}), \, 0 \leq t^1, \dots, t^n \leq T \notag\\
& (X^1_{0}, \dots, X^i_{0}, \dots, X^n_{0}) \sim p_{\text{data}}(\bx)
\end{align}
where $\boldsymbol{t} = (t^1, \dots, t^n)$, $\bx = (x^1, \dots, x^n)$.  We consider the $i$-th augmented process $\overline{X^{i}} = (X^i_{t_{i}}, t_{i}) \in \mathcal{X}^{i}_{*} = \mathcal{X}^i \times [0, +\infty)$ and we denote its generator as $\mathcal{L}_{X^i}$. We slightly abuse the notation and define the application of $\mathcal{L}_{X^i}$ to a multivariable test function as the following,
\begin{align*}
   \cL_{X^i} f(\bx) = \lim\limits_{t^i \to 0} \dfrac{\E[f(x^1, {\tiny\dots} X^i_{t^i}, {\tiny\dots}, x^n) | X_0^i = x^i] - f(\bx) }{t^i} 
\end{align*}

We assume that $X^i_{t^i}$ are independent Markov processes when conditioned on initial conditions. We want to emphasize that the design of this forward process \cref{eq:joint-forward} is not only for injecting probabilistically independent `noises' into each modality. More importantly, it allows each modality to be noised in an \textbf{asynchronous} way. 

To fully characterize and understand this forward process as a whole, we need to learn the full joint marginal of $\boldsymbol{X}_{\boldsymbol{t}}$, which we denote as $p(\bx, \bt)$ and should be understood as the joint distribution of $X^1, \dots, X^n$ at time $t^1, \dots, t^n$. To visualize this idea, we demonstrate the forward process with two independent time variables in \cref{fig:promo}.

To `invert' this forward process for generative modeling purposes, we will need to learn information from the forward process, similar to the case of unimodal diffusion models where one performs score matching. Extending the framework of \citep{benton2022denoising}, we introduce the following generalized explicit score matching loss (GESM) for learning the full marginal $p(\bx, \bt)$. 
\begin{align*}
&\cI_{\mathrm{GESM}} = \\
& \underset{\bt, \bx_{\bt} \sim p(\cdot, \bt)}{\mathbb{E}} \Bigg[\sum_{i=0}^n\dfrac{\cL_{X^i} (p/\beta_\theta)(\bx_{\bt}, \bt)}{(p/\beta_\theta)(\bx_{\bt},\bt)} - \cL_{X^i} \log (p/\beta_\theta)(\bx_{\bt},\bt) \Bigg]
\end{align*}
here $\beta_{\theta}: \mathcal{X}^1_{*} \times \cdots \times \mathcal{X}^n_{*} \to \mathbb{R}^{+}$ is our parameterized unnormalized distribution. We have the following important properties that characterize the optimizer of $\cI_{\mathrm{GESM}}$
\begin{theorem}
\label{thm:gesm}
$\cI_{\mathrm{GESM}} \geq 0$, with equality reached when $\beta_\theta(\bx, \bt) \propto p(\bx, \bt)$.
\end{theorem}
\cref{thm:gesm} states that the minimizer of $\cI_{\mathrm{GESM}}$ is $p(\bx,\bt)$ up to a multiplicative constant. In practice, we often do not directly model $\beta_{\theta}$, but instead parameterize its score functions, which are invariant to multiplicative constants. For example, in continuous diffusion, one often parameterizes $\nabla \log \beta_{\theta}$. This makes the optimization of $\cI_{\mathrm{GESM}}$ a well-defined problem with a unique minimizer in terms of score learning. In practice, one can't evaluate $\cI_{\mathrm{GESM}}$ as it's intractable due to the true marginals $p$ being unavailable a priori. Luckily, one can efficiently compute the following denoising and implicit variants of $\cI_{\mathrm{GESM}}$ for learning purposes.
\begin{theorem}\label{thm:loss-equiv} $\cI_{\mathrm{GESM}}$, $\cI_{\mathrm{GDSM}}$ and $\cI_{\mathrm{GISM}}$ are equivalent up to constants, where
\begin{align*}
& \cI_{\mathrm{GDSM}} = \\
&\underset{\substack{\bt, p_0, \\
p_{\bt|0}}}{\mathbb{E}} \Bigg[\sum_{i=1}^n \dfrac{\cL_{X^i} (p_{\bt|0}/\beta_\theta)(\bx_{\bt}, \bt)}{(p_{\bt|0}/\beta_\theta)(\bx_{\bt},\bt)} - \cL_{X^i} \log (p_{\bt|0}/\beta_\theta)(\bx_{\bt},\bt) \Bigg] \\
& \cI_{\mathrm{GISM}} = \underset{\bt, p_{\bt}}{\mathbb{E}} \Bigg[\sum_{i=1}^n \dfrac{\cL_{X^i}^* \beta_\theta(\bx_{\bt}, \bt)}{\beta_\theta(\bx_{\bt},\bt)} - \cL_{X^i}^* \log (\beta_\theta)(\bx_{\bt},\bt) \Bigg]
\end{align*}
\end{theorem}

\subsection{Continuous-Discrete Multimodal Diffusion}
\label{sec:cont_discrete_diffusion}
We consider a distribution $p_{\text{data}}(x,y)$ where $x \in \R^d$, $y\in \mathbb{X}$, where $\mathbb{X}$ is a finite space. Important applications include text-image joint generation, mixed-type tabular data synthesis, etc. In this case, the natural choices on each state space would be continuous diffusion on $\mathbb{R}^{d}$ and discrete diffusion on $\mathbb{X}$. This results in the following forward process,
\begin{align}
\label{eq:forward}
   \begin{cases}
    \rd X_t = f(X_t, t) \rd t + g(t) \rd B_t\\
    Y_s \sim \operatorname{CTMC}(Q_{s}), \; (X_0, Y_0) \sim p_{\text{data}}(x, y) 
\end{cases}
\end{align}
Let's denote the joint marginal of $(X_t, Y_s)$ as $p_{t,s}(X_t, Y_s)$. with these choices of Markov processes, $\mathcal{L}_{X} = \partial_t + f \cdot \nabla + \frac{1}{2} g^2(t) \Delta$, and $\mathcal{L}_{Y} = \partial_s + Q_s$. Therefore, we can compute the generalized denoising score matching objective as,
\begin{proposition}\label{prop:cont_disc_simp}
For forward process \cref{eq:forward}, $\cI_{\mathrm{GDSM}}$ is equivalent to the following objective,
\begin{align*}
& \underset{\substack{t, s, x_0, y_0 \sim p_{0} \\ x_t, y_s \sim p_{t, s|0}}}{\mathbb{E}}\Big[ \frac{1}{2} g^{2}(t) \| s_{\theta}^{X} - \nabla \log p_{t}(x_t | x_0) \|^2 + \\
& \qquad \sum_{y \neq y_s}Q_s(y_s, y) \Big((\boldsymbol{s}^{Y}_{\theta})_{y} - \frac{p_{s}(y|y_0)}{p_s(y_s|y_0)} \log (\boldsymbol{s}^{Y}_{\theta})_y  \Big) \Big]
\end{align*}
where $\boldsymbol{s}_{\theta}^{X}$,$\boldsymbol{s}_{\theta}^{Y}$ is the learned continuous/discrete score.
\end{proposition}
Importantly, \cref{prop:cont_disc_simp} shows an amazing result that score functions of \textbf{multimodal joint marginal} $p(\cdot, \cdot, t, s)$ can be learned through score matching with \textbf{unimodal conditional score} for each modality. This is a non-trivial result as $\cI_{\mathrm{GDSM}}$ suggests that, to perform multimodal generation, we need to match the score network with conditional scores of the joint distribution such as $\nabla \log p_{t,s}(x_t, y_s | x_0, y_0)$ instead of $\nabla \log p_{t}(x_t | x_0)$. However, thanks to Bayes' rule and the design of independent noise injection per modality, the two conditional scores are \textbf{identical}. Thus, while learning multimodal diffusion models may seem as simple as jointly optimizing a sum of unimodal diffusion model training objectives, the theoretical support for such a naive approach is grounded much deeper. 

With the learned score functions, we can `invert' the forward process for generative purposes, which is stated in \cref{prop:backward_text_image}. 

\begin{proposition}
\label{prop:backward_text_image}
The following process $(X_t, Y_s)$ has marginal distribution equals to $p(x,y, T-t, T-s)$,
\begin{align*}
   \begin{cases}
    \begin{aligned}
        \rd X_t &= - f(X_t, T-t) + g^2(T-t) \nabla \log \overleftarrow{p}_{t, s}(X_t, Y_s) \rd t \\
        & \quad + g(T-t) \rd B_t
    \end{aligned} \\
    Y_s \sim \operatorname{CTMC}\big(\overleftarrow{Q}(X_t, t, s)\big), \; (X_0, Y_0) \sim p(x, y, T, T)
\end{cases}
\end{align*}
where $\overleftarrow{p}_{t,s} = p_{T-t, T-s}$, $\overleftarrow{Q}(X_t, t, s)$ is a rate matrix with $y', y$ entry being $\frac{\overleftarrow{p}_{t,s}(X_t, y')}{\overleftarrow{p}_{t,s}(X_t, y)}(Q_{T-s})_{yy'}$ when $y' \neq y$.
\end{proposition}
We note that this `backward process' is not the time reversal of the forward process in a strict sense, as it involves introducing multiple time variables into the system. However, this enables us to design versatile conditional and unconditional generative sampling algorithms by choosing different value combinations of times $(t,s)$. We defer the detailed discussion of design choices and sampling algorithms for continuous-discrete multimodal diffusion to \cref{app:cont_disc}.

\noindent \textbf{Multimodal unconditional generation. } To jointly generate clean data distributed as $p_{\text{data}}(x,y)$, we introduce a simulation time variable $u \in [0,T]$, and pick a time parameterization $t = \alpha_1(u)$, $s = \alpha_2(u)$ such as $\alpha_{i}: [0,T] \to [0,T]$ is continuous, non-decreasing with $\alpha_i(0) = 0, \alpha_i(T) = T$ for $i = 1, 2$. With this time-reparameterization for both time variables, we make the backward process a valid, ready-to-simulate process. This enables us to start from pure noise and generate samples from the data distribution unconditionally for both modalities. We can also choose a singular time re-parameterization so that the simulation amounts to a 2-stage approach for multimodal generation, where we first generate one modality and sample the rest conditionally based on the generated sample.

\noindent \textbf{Unimodal conditional generation. } This framework with decoupled time also enables conditional generation of modalties by simulating the associated backward process. We have the following simple but important observations,
\begin{itemize}
\vspace{-3pt}
    \item Given a partially noisy text $Y_{s}$ and its noise level $s$, simulating the $X$ backward dynamics generates a sample $X_{T} \sim p_{\text{data}}(x | Y_{s}, s)$
    \item Given a partially noisy image $X_t$ and its noise level $t$, simulating the $Y$ backward dynamics generates a sample $Y_{T} \sim p_{\text{data}}(y | X_{t}, t)$.
\end{itemize}
Note that the choices of $t$ or $s$ in the conditioning are not restricted. When picking $t$ or $s$ as $T$, this is equivalent to single-modality unconditional generation, as $X_{T}$ or $Y_{T}$ are pure noise. When picking $t$ or $s$ as $0$, this is equivalent to conditional generation, as $X_0$ or $Y_0$ are clean data samples. More interestingly, when picking $0< s, t < T$ as conditions, we generate samples based on partially noised conditions. This gives rise to the following new guidance mechanism for enhancing generation quality.

\begin{figure}[!t]
    \centering
    \includegraphics[width=0.45\textwidth]{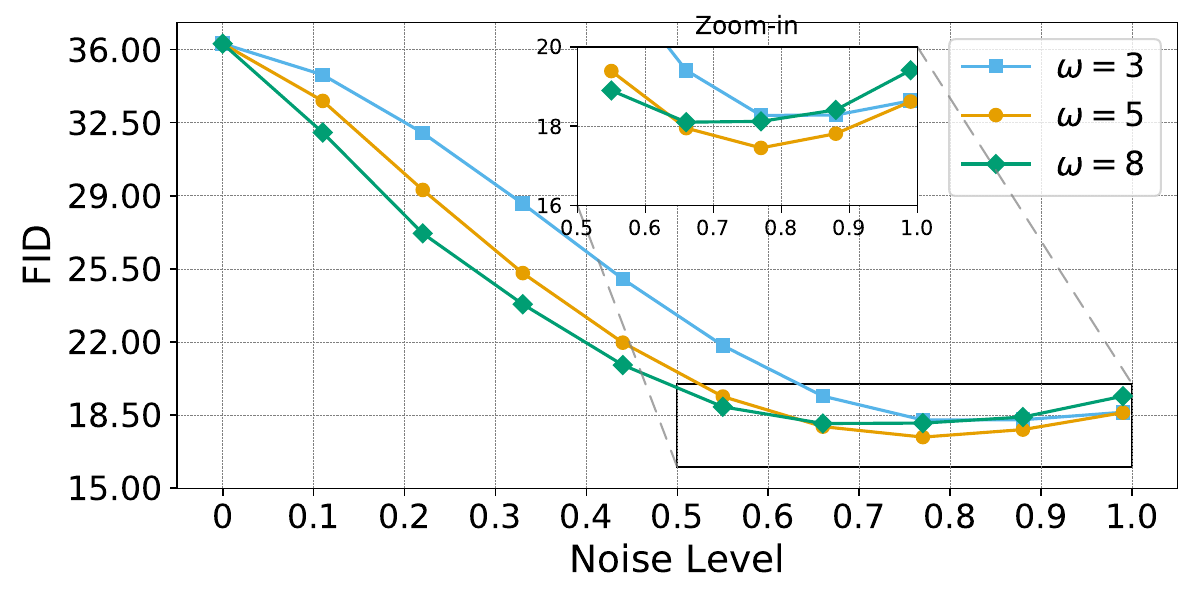}
    \vspace{-10pt}
\caption{Performance of noisy guidance on MS-COCO FID-$10$K. We note that using partially noised conditions results in a better performance. A guidance interval of $t\in [0.3,0.8]$ was used.}
\label{fig:noisy-guidance}
\vspace{-2em}
\end{figure}
\begin{table*}[!t]
    \centering
    \vspace{-1.2em}
    \caption{Results on the \textbf{text to image conditional generation} on MS-COCO. We mark the extra encoders leveraged by each model with the corresponding sizes and types. SR: super resolution, TE: text encoder, VAE: variational autoencoder, VE: visual encoder, VQ-GAN: Vector Quantized GAN, VQ-VAE: vector-quantized variational autoencoder.}\label{tab:fid_ms}
    \resizebox{\linewidth}{!}{
    \begin{tabular}{lcccc}
    \toprule
        Model & FID & Number of Images & \#Params & Extra Encoders\\
    \midrule
    Models for Text-to-Image generation only\\ 
    \midrule
        \quad DALL-E 2~\citep{ramesh2022hierarchical} & 10.39 & 650M & 6.5B & 123M (TE) + 700M (SR) \\
        \quad Imagen~\citep{saharia2022photorealistic} & 7.27 & 860M & 3B & 4.6B (TE) + 600M (SR) \\
        \quad Stable Diffusion \citep{rombach2022high} & 12.63 & 400M & 1.45B &  123M (TE) + 83M (VAE) \\
        \quad PixArt-$\alpha$ XL/2 \citep{chen2023pixart} & 7.32 & 25M & 600M & 123M (TE) + 83M (VAE) \\
        \quad MMDiT-improved \citep{ifriqi2024improved} & 6.79 & 12M &  600M & 123M (TE) + 83M (VAE)  \\
    \midrule
    Models for multimodal generation and understanding\\
    \midrule
        \quad Show-o \citep{xie2024show} & 9.24 & 35M & 1.3B &  115M (VE) + 307M (VQ-VAE) \\
        \quad Transfusion \citep{zhou2024transfusion} & 6.78 & 692M & 7B  & 86M (VAE) \\
        \quad Chameleon \citep{team2024chameleon} & 26.7 & 600M & 7B  & 307M (VQ-GAN)\\
        \quad JetFormer \citep{tschannen2024jetformer} & 20.86 & 1B & 2.75B &  ---\\
    \midrule
    Models for multimodal generation only \\
    \midrule
        \quad Versatile Diffusion \citep{xu2023versatile} & 11.10 & 400M & 1.45B &  123M (TE) + 83M (VAE) + 110M (TE) \\
        \quad UniD3 \citep{hu2022unified} & 25.11 & 592K & 600M & 123M (TE) + 307M (VQ-GAN) \\
    \midrule
        \quad Our model & 16.16 & 12M & 481M  & 83M (VAE) \\
    \bottomrule
    \end{tabular}}
    \vspace{-1em}
\label{tab:fid}
\end{table*}

\subsection{Noisy Guidance} 
Guidance techniques have been a core component in modern diffusion models for improving generation quality \citep[e.g.,][]{dhariwal2021diffusion, ho2021classifierfree, kynkaanniemi2024applying, li2024self}. For continuous diffusion models, with strength $\omega$, the classifier-free guidance is obtained by interpolating the unconditional and conditional score functions,
\begin{align}
\label{eq:cfg}
    \omega s_{\theta}(x_t, t, c) + (1 - \omega) s_{\theta}(x_t, t, \emptyset)
\end{align}
One perspective to understand the effectiveness of guidance methods is to view the unconditional score function as a conditional model with a fully-noised condition input, and the interpolation effectively serves as a correction of the conditional scores. However, the unconditional score might not be the best choice of correctors, as it causes an excessive trade-off between fidelity and diversity, resulting in a significant loss in the latter \citep{karras2024guiding}. This raises an interesting question about whether CFG can be improved by finding a better alternative to the unconditional score in \cref{eq:cfg}. In fact, within our framework, we notice that
\begin{align*}
s_{\theta}^{X}(x_t, y_s, t, s) \approx \nabla_{x} \log p_{t,s}(x_t, y_s) = \nabla_{x} \log p(x_t, t|y_s, s)
\end{align*}
The last equality results from Bayes' theorem, and it shows that our model in fact learns conditional scores \textbf{at all noise levels}. Leveraging this fact, we propose a new form of guidance named \textbf{noisy guidance}, where the unconditional score in \eqref{eq:cfg} is replaced with a class of conditional models with conditions noised to different levels:
\begin{align} \label{eq:noisy-guidance}
   &\omega s_\theta(x_t, y_s,t, s) + (1-\omega) s_\theta(x_t, y_\sigma, t, \sigma)  & \sigma > s
\end{align}
We note that the noisy guidance framework is a more general one, as it recovers the vanilla setting of CFG when $s = 0$ and $ \sigma = T$. 
The scenario of conditional generation corresponds to $s = 0$ (a clean condition $y_0$ is given), and we choose $T > \sigma > 0$ to improve generation quality with a \textbf{partially conditioned} guiding model. More interestingly, noisy guidance can even be applied to \textbf{unconditional generation} in an unsupervised way where $s > 0$. In this case, while $s$ is changing (since we are also generating $y_s$), we can still apply guidance in this process by adaptively picking $\sigma$ as long as $\sigma > s$. These observations indicate the power and robustness of noisy guidance as a by-product of our proposed multimodal diffusion model learning framework.

\begin{table*}[!t]
\begin{center}
\vspace{-1.2em}
\caption{Performance on the \textbf{Trend} metric in percentage (\%). Higher values indicate better performance. Best performance in \textbf{bold}. Second best in \underline{underline}.}
\label{tab:trend}
\scalebox{0.85}{
\begin{tabular}{@{}cc|cccccc@{}}
\toprule

Methods &\#Parameters & Adult & Default & Shoppers & Magic & Beijing & News \\ \midrule
GOGGLE \citep{liu2023goggle} &$\sim$ 5.6M &54.71 &78.06 &76.10 &90.53 &54.06 &76.81 \\
STaSy \citep{kim2022stasy} &$\sim$ 10.3M &85.49\scriptsize{$\pm$0.25} &94.04\scriptsize{$\pm$0.26} &91.51\scriptsize{$\pm$0.15} &93.39\scriptsize{$\pm$0.53} &92.00\scriptsize{$\pm$0.10} &96.93\scriptsize{$\pm$0.04} \\
CoDi \citep{lee2023codi} &$\sim$ 25.0M &77.51\scriptsize{$\pm$0.08} &31.59\scriptsize{$\pm$0.05} &82.22\scriptsize{$\pm$0.11} &93.47\scriptsize{$\pm$0.25} &92.93\scriptsize{$\pm$0.15} &88.90\scriptsize{$\pm$0.01} \\
TabDDPM \citep{kotelnikov2023tabddpm} &$\sim$ 11.7M &96.99\scriptsize{$\pm$0.25} &95.11\scriptsize{$\pm$0.10} &93.39\scriptsize{$\pm$0.16} &98.30\scriptsize{$\pm$0.22} &97.20\scriptsize{$\pm$0.09} &86.84\scriptsize{$\pm$0.11} \\ 
T{\small AB}S{\small YN}  \citep{zhang2023mixed} &$\sim$ 10.7M  &\underline{98.46}\scriptsize{$\pm$0.27} &\textbf{97.95}\scriptsize{$\pm$0.12} &\underline{97.93}\scriptsize{$\pm$0.21} &\underline{98.94}\scriptsize{$\pm$0.31} &\textbf{97.76}\scriptsize{$\pm$0.28} &\underline{98.56}\scriptsize{$\pm$0.03}  \\
T{\small AB}S{\small YN}  (reproduced) &$\sim$ 10.7M &98.29\scriptsize{$\pm$0.22} &95.25 \scriptsize{$\pm$0.51} &97.82\scriptsize{$\pm$0.14} &\textbf{99.16}\scriptsize{$\pm$0.16} &94.86\scriptsize{$\pm$0.34} &98.52\scriptsize{$\pm$0.09} \\
\midrule
Our model &$\sim$ \textbf{64K} &\textbf{98.75}\scriptsize{$\pm$0.09} &\underline{96.00}\scriptsize{$\pm$1.23} &\textbf{98.24}\scriptsize{$\pm$0.13} & 98.85\scriptsize{$\pm$0.42} &\underline{97.42}\scriptsize{$\pm$0.11} &\textbf{98.57}\scriptsize{$\pm$0.16}
\\ \bottomrule
\end{tabular}
}
\end{center}
\vspace{-0.5cm}

\end{table*}

\begin{table*}[!t]
\caption{Performance on the \textbf{MLE} metric. Higher values in AUC and lower values in RMSE indicate better testing performance. Best performance in \textbf{bold}. Second best in \underline{underline}.}
\label{tab:MLE}
\begin{center}
\scalebox{0.85}{
\begin{tabular}{@{}cc|cccccc@{}}
\toprule

Methods &\#Parameters & Adult & Default & Shoppers & Magic & Beijing & News \\ 
& & (AUC$\uparrow$) & (AUC$\uparrow$) & (AUC$\uparrow$) & (AUC$\uparrow$) & (RMSE$\downarrow$) & (RMSE$\downarrow$) \\ \midrule
GOGGLE \citep{liu2023goggle} &$\sim$ 5.6M &.778\scriptsize{$\pm$0.012} &.584\scriptsize{$\pm$0.005} &.658\scriptsize{$\pm$0.052} &.654\scriptsize{$\pm$0.024} &1.090\scriptsize{$\pm$0.025} &.877\scriptsize{$\pm$0.002} \\
STaSy \citep{kim2022stasy} &$\sim$ 10.3M &.906\scriptsize{$\pm$0.001} &.752\scriptsize{$\pm$0.006} &.914\scriptsize{$\pm$0.005} &.934\scriptsize{$\pm$0.003} &.656\scriptsize{$\pm$0.014} &.871\scriptsize{$\pm$0.002} \\
CoDi \citep{lee2023codi} &$\sim$ 25.0M &.871\scriptsize{$\pm$0.006} &.525\scriptsize{$\pm$0.006} &.865\scriptsize{$\pm$0.006} &.932\scriptsize{$\pm$0.003} &.818\scriptsize{$\pm$0.021} &1.21\scriptsize{$\pm$0.005} \\
TabDDPM \citep{kotelnikov2023tabddpm} &$\sim$ 11.7M &.907\scriptsize{$\pm$0.001} &.758\scriptsize{$\pm$0.004} &.918\scriptsize{$\pm$0.005} &.935\scriptsize{$\pm$0.003} &.592\scriptsize{$\pm$0.011} &4.86\scriptsize{$\pm$3.04} \\
T{\small AB}S{\small YN} \citep{zhang2023mixed}  &$\sim$ 10.7M &\textbf{.915} \scriptsize{$\pm$0.002} &\textbf{.764}\scriptsize{$\pm$0.004} &\underline{.920}\scriptsize{$\pm$0.005} &.938\scriptsize{$\pm$0.002} &\underline{.582}\scriptsize{$\pm$0.008} &\underline{.861}\scriptsize{$\pm$0.027}  \\
T{\small AB}S{\small YN}  (reproduced)  &$\sim$ 10.7M &.910\scriptsize{$\pm$0.001} &.755\scriptsize{$\pm$0.004} &.916\scriptsize{$\pm$0.004} &\underline{.939}\scriptsize{$\pm$0.003} &.655\scriptsize{$\pm$0.012} &\textbf{.851}\scriptsize{$\pm$0.024} \\ 
\midrule
Our model &$\sim$ \textbf{64K} &\textbf{.915}\scriptsize{$\pm$0.001} &\textbf{.764}\scriptsize{$\pm$0.002} &\textbf{.924}\scriptsize{$\pm$0.003} &\textbf{.941}\scriptsize{$\pm$0.002} &\textbf{.543}\scriptsize{$\pm$0.012} &.864\scriptsize{$\pm$0.021}
\\ \bottomrule
\end{tabular}
}
\end{center}
\vspace{-0.5cm}
\end{table*}

To showcase the performance of noisy guidance, we evaluate FID-10K on MS-COCO \citep{coco} using different values of $\sigma \in [0,1]$. We find that partially denoising the caption results in an improved FID, as shown in \cref{fig:noisy-guidance}. The superior performance of noisy guidance is possibly due to using a \textbf{partially corrupted version} of the conditional model as the guiding model. This aligns closely with the idea of Autoguidance proposed in \citet{karras2024guiding}, which utilizes an under-trained, smaller model as the guiding model instead of the unconditional ones. Similar to Autoguidance, noisy guidance seeks to identify and reduce the errors made by the conditional score model by measuring its difference to the partially conditioned one, boosting the generation performance. Finally, we remark that \cref{eq:noisy-guidance} is only a special case applied to continuous diffusion, and a similar idea can be adapted to other modalities, such as discrete diffusion guidance \citep{nisonoff2024unlocking, schiff2024simple}. 

\vspace{-0.8em}
\section{Experiments}
To demonstrate the effectiveness and relevance of our framework for training multimodal diffusion models, we consider two tasks: text-image generation and mixed-type tabular data synthesis, both of which are accomplished using the continuous-discrete Multimodal diffusion discussed in \cref{sec:cont_discrete_diffusion}. The results for text-image generation are presented in \cref{sec:text_image} and the results for tabular data synthesis are presented in \cref{sec:tabular}. We also include results for combining Riemannian and discrete diffusion in \cref{app:riem_discrete} to demonstrate the generality of the framework.
\subsection{Text-Image Generation}
\label{sec:text_image}
\textbf{Architecture } We design a new score network backbone for this task based on the celebrated success of Diffusion Transformer (DiT) \citep{peebles2023scalable} and Multimodal Diffusion Transformer (MMDiT) \citep{esser2024scaling}. We first process the inputted (noisy) images and texts by passing them through an MMDiT with a per-modality unique time conditioning. MMDiT's remarkable strength in modeling cross-modal interaction, as well as allowing independent conditioning for each modality, makes it ideal for the backbone. The tokens then undergo unimodal DiTs for a more refined learning process. We present a comprehensive diagram of the backbone in \cref{fig:architecture} and refer the reader to \cref{app:text_image} for further details.

\noindent \textbf{Datasets} We train on the SAM-LLaVA dataset introduced by \citet{chen2023pixart}. This dataset is constructed by adding captions to the Segment Anything (SAM) \citep{kirillov2023segment} using LLaVA \citep{liu2024visual}, which results in rich and diverse captions. However, it suffers from hallucinations of LLava. For example, many colored images are described as being black and white. Following \cite{chen2023pixart}, we tokenize each caption with a length of $120$ tokens.

\noindent \textbf{Training \& Evaluation} We train our model using a multi-stage training strategy. We kindly refer readers to \cref{app:text_image} for more details. We evaluate FID-$30$K on MS-COCO \citep{coco}. Compared with SAM-LLaVA, MS-COCO comes with much shorter captions. To address this distribution shift in caption length between training and inference time, we draw inspiration from \citep{ifriqi2024improved} and replicate the text to increase the caption size. We also limit the number of tokens to $40$ during this evaluation. We compare our results to other methods in \cref{tab:fid_ms}. In terms of text-to-image (T2I), our methods produce similar performance compared to many other industrial-level models with larger model sizes. Notably, our model is trained on fewer samples, features a significantly smaller backbone, and does not utilize extra foundation models to aid multimodal representation learning. This reflects both the efficiency and the effectiveness of our proposed approach. We also present qualitative samples in \cref{fig:qualitative-samples}. For evaluation of image-to-text (I2T) and joint generation, please kindly see \cref{app:text-image}.

\subsection{Mixed-type Tabular Data Synthesis}
\label{sec:tabular}

\noindent \textbf{Architecture } We devise a score network based on DiT for this task, where both discrete and continuous tabular data are fed into the transformer after simple dimension rescaling. Our design aims at achieving early fusion of both modalities for more efficient learning. See \cref{app:tabular_training} for details.

\noindent \textbf{Datasets } We experiment on $6$ real-world tabular datasets acquired from UCI Machine Learning Repository\footnote{https://archive.ics.uci.edu/}. Every dataset contains columns of numerical or categorical features, associated with binary classification tasks or regression tasks. Detailed descriptions of datasets are in \cref{app:tabular_description}.

\noindent\textbf{Training \& Evaluation } We describe the score model architecture and training settings in \cref{app:tabular_training}. For evaluation, we follow the same setting as \citet{zhang2023mixed}. We evaluate the Trend, Machine Learning Efficiency (MLE), Shape, Precision, and Recall of the generated data. We present results on \textbf{Trend} in \cref{tab:trend} and \textbf{MLE} in \cref{tab:MLE}. For more results on other metrics, please see \cref{app:tabular_eval}. All experiments are repeated $20$ times for robustness.

As evident from \cref{tab:trend} and \cref{tab:MLE}, our model performs either the best or second-best, with a negligible gap, among most datasets. Notably, our model has only $\sim$ 64K parameters, which is \textbf{100 to 200 times smaller} than models used in other methods. The significant reduction in model size stems from the fact that our method operates natively on the state space of mixed-type tabular data, eliminating parameter-heavy encoders like VAE. Our newly designed transformer-based score network leverages the concept of early fusion \citep{team2024chameleon}, which also enhances parameter efficiency. This network learns a joint embedding between modalities starting from the first attention layer and is more parameter-efficient. These results again showcase the efficiency and advantage of our proposed multimodal diffusion model framework.
\vspace{-1em}
\section{More Information on Related Works}

\noindent \textbf{Multimodal Generative Models. } Various works in the literature have attempted joint generation of multimodal data. Many existing methods approach the task by leveraging autoregressive models by first tokenizing multimodal data into discrete tokens and then generating them sequentially \citep{team2024chameleon, xie2024show, zhou2024transfusion, tschannen2024jetformer, lu2024unified, ge2024seed, wu2024janus, wang2024emu3}. Another portion of the algorithms is built on the versatile capability of diffusion/flow-based methods to generate latent representations of multimodal data \citep{lee2023codi, bao2023one, ruan2023mm, hu2022unified, zhang2023mixed, kim2024versatile, chen2024diffusion, li2024omniflow, swerdlow2025unified, hayes2025simulating}. One thing in common among the aforementioned approaches is that they all extensively utilize tokenizers or encoders to produce a unimodal latent space for multimodal data, which are not modular and heavily tailored to specific applications with little theoretical support. In contrast to these works in the literature, we propose a general multimodal diffusion learning framework in this paper, which is flexible for generating data on arbitrary state spaces. This is achieved by minimizing the need for external, modality-specific tokenizers and encoders, while keeping the generation in the native spaces of the targeted data.

\textbf{Decoupled Time Variables. } It's worth noting that the decoupled time design, essential to our proposed framework, has in fact been explored by many application-driven works in the literature, such as UniDiffuser \citep{bao2023all}, MultiFlow \citep{campbell2024generative}, AVDiT \citep{kim2024versatile}, and OmniFlow \citep{li2024omniflow}. While these methods all leverage the multiple time variable design to achieve any-to-any generation, they mostly consider this design as a trick and do not investigate it from an algorithmic perspective. Our work contributes to this literature by deriving the unified training objective and backward generative process in the presence of multiple time variables, providing theoretical justification for the validity of this design. To the best of our knowledge, our work is the first to formalize this idea and generalize it to an arbitrary number of modalities. 

A highly relevant work worth discussing in more detail is UniDiffuser \citep{bao2023one}, which addresses the same text-image joint generation task considered in this paper while also utilizing decoupled time variables. A fundamental difference in algorithm design is that UniDiffuser purely relies on continuous diffusion models in the CLIP latent space, which is shared by both text and images, whereas our proposed continuous-discrete diffusion operates natively on the product space of Euclidean and finite-state spaces. 

\noindent \textbf{Diffusion Models in General State Spaces.} Theoretical results have demonstrated that diffusion models can be generalized to denoising Markov models, a class of generative models constructed based on the notion of Markov processes \citep{benton2022denoising, ren2025unified}. These works considered unimodal diffusions on general state spaces with a single time variable. Our work extends the framework of denoising Markov models by incorporating multimodal diffusion models on the product of different state spaces, as well as multiple time variables. In the literature, there are also MultiFlow \citep{campbell2024generative} and Generator Matching \citep{holderrieth2024generator}, which are multimodal extensions of flow-based methods in native state spaces. Their algorithm construction specifically focused on the task of protein sequence-structure co-generation. In contrast, we present a general recipe for multimodal diffusion models that does not initially target specific tasks, despite empirically validating the proposed framework on two examples: text-image generation and mix-type tabular data synthesis, leveraging our newly proposed Continuous-Discrete Diffusion.
\vspace{-1em}
\section{Conclusions, Limitations and Future Works}
We propose a novel framework for constructing multimodal diffusion models on general state spaces. We experiment on text-image and tabular data generation, and our approach achieves competitive performances with a significantly smaller model size. One limitation of this work is that we didn't explore the possibility of utilizing pretrained unimodal diffusion models as initialization of multimodal diffusion training, which could further boost training efficiency. We leave this as a future direction for investigation.

\section*{Acknowledgements}
KR, YZ and MT are grateful for partial supports by NSF Grant DMS-1847802, Cullen-Peck Scholarship, and Emory-GT AI.Humanity Award. FY is grateful for partial support by UAlbany-IBM research funding. This research was supported in part by the AI Computing Cluster at the University at Albany.

\section*{Impact Statement}
This paper presents work aimed at advancing the field of Machine Learning. There are many potential societal consequences of our work, none of which we feel must be specifically highlighted here.

\bibliography{bibliography}
\bibliographystyle{icml2025}

\newpage
\appendix
\onecolumn

\input{proofs}

\input{cont_disc}

\input{diff_details}

\input{tabular}
\input{toy_riemann}

\section{Additional Numerical Results for Text-Image Generation} \label{app:text-image}

\paragraph{CLIP Similarity} We generate $5000$ samples and evaluate the CLIP similarity between the text and image. For this evaluation, we use CLIP-ViT-large-patch14 and we limit the captions to $77$ tokens. We use our sampling default values during this task. We obtain a CLIP score of $\textbf{18.46}$ for text-to-image generation, $\textbf{17.44}$ for image-to-text generation, and $\textbf{17.57}$ for image-text joint generation. 

\paragraph{Generated examples visualization} We display non-cherry-picking generated examples in all three scenarios in \cref{fig:I2T}, \cref{fig:T2I}, \cref{fig:IT}.
\begin{figure*}[h!]
    \centering
    \includegraphics[width=\linewidth]{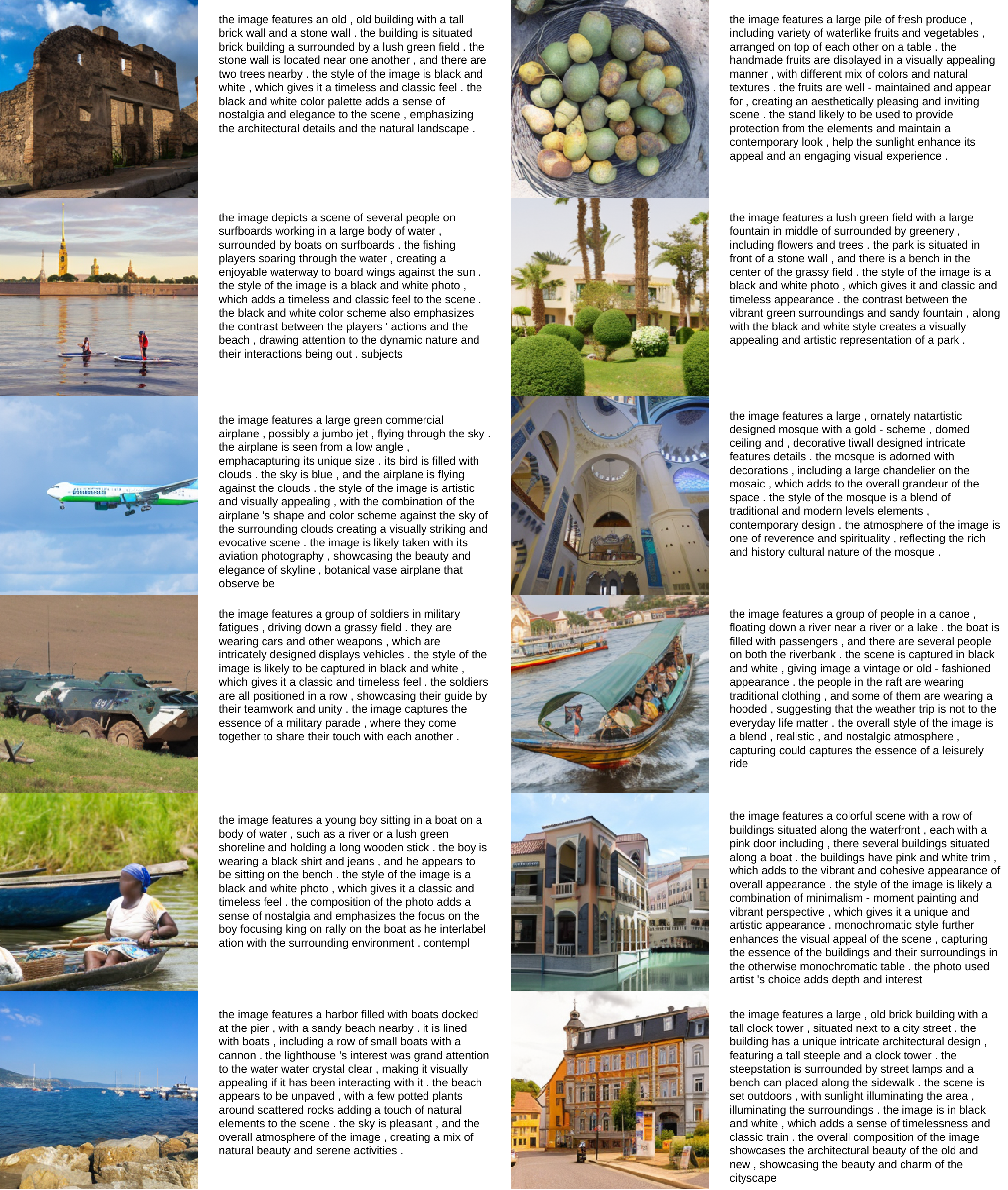}
    \caption{Visualization of texts generated conditioning on the images.}
    \label{fig:I2T}
\end{figure*}
\begin{figure*}[h!]
    \centering
    \includegraphics[width=\linewidth]{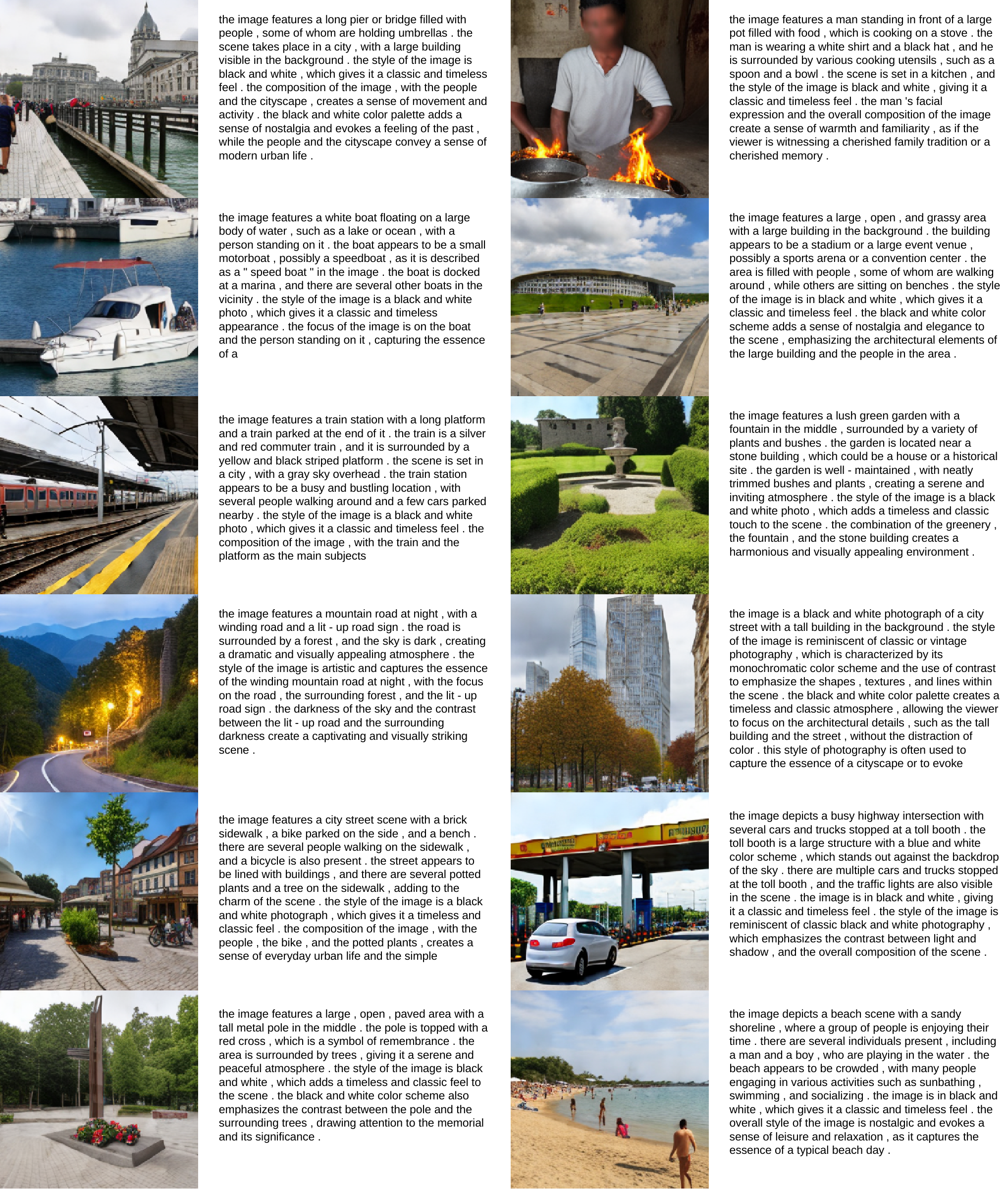}
    \caption{Visualization of images generated conditioning on the text caption.}
    \label{fig:T2I}
\end{figure*}
\begin{figure*}[h!]
    \centering
    \includegraphics[width=\linewidth]{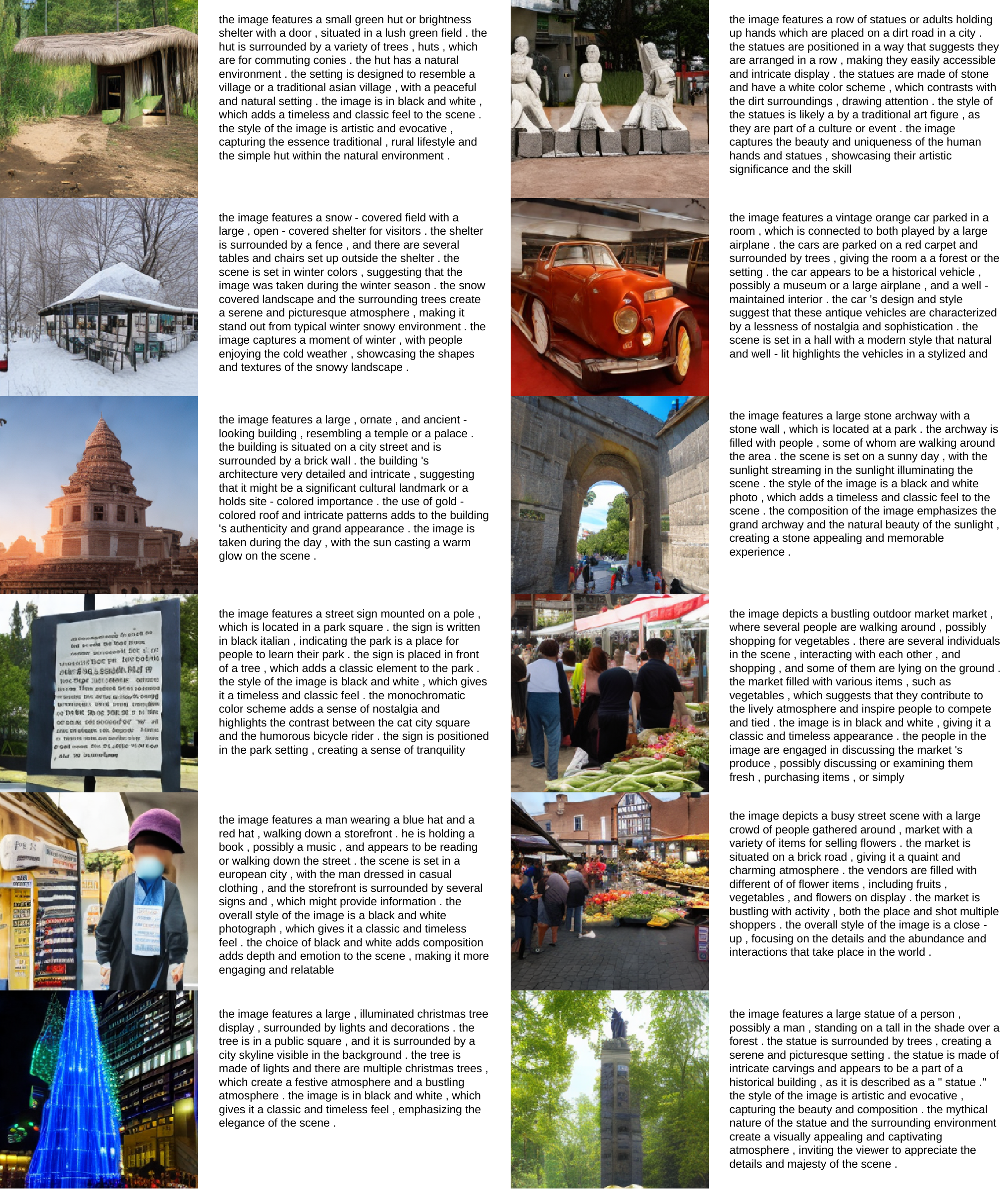}
    \caption{Visualization of text-image pairs generated jointly and unconditionally.}
    \label{fig:IT}
\end{figure*}

\end{document}

%% file: macros.tex
\DeclareMathOperator{\Law}{Law}
\newcommand{\rd}{{\mathrm{d}}}

\newcommand{\bt}{\boldsymbol{t}}
\newcommand{\bx}{\boldsymbol{x}}

\newcommand{\R}{\mathbb{R}}
\renewcommand{\P}{\mathbb{P}}
\newcommand{\E}{\mathop{\mathbb{E}}}

\newcommand{\cI}{{\mathcal{I}}}
\newcommand{\cL}{{\mathcal{L}}}
\newcommand{\cX}{{\mathcal{X}}}

\newcommand{\inner}[2]{\langle #1, #2 \rangle}

\newcommand{\yh}{\hat{y}}

\newcommand{\pbeta}{p_{\bt|0}/\beta_\theta}

%% file: proofs.tex
\section{Proofs}
\label{app:proof}

\subsection{Proof of Theorem \ref{thm:gesm}}
The following proof follows a similar idea as is presented in \citet{benton2022denoising}: 
\begin{proof}
By Jensen's inequality, we have:
\begin{align*}
\log\left(\E[f(x^1,\dots, X_{t^i}^i, \dots, x^n) ~\big|~X_0^i=x^i] \right) \leq \E[\log(f)(x^1,\dots, X_{t^i}^i, \dots, x^n) ~\big|~X_0^i=x^i]
\end{align*}
This implies the following inequality:
\begin{align*}
  \underbrace{\frac{\log\left(\E[f(x^1,\dots, X_{t^i}^i, \dots, x^n) ~\big|~X_0^i=x^i] \right) - \log f(\bx)}{t^i}}_{\mathrm{LHS}}  \leq \underbrace{\frac{\E[\log(f)(x^1,\dots, X_{t^i}^i, \dots x^n) ~\big|~X_0^i=x^i] - \log f(\bx)}{t^i}}_{\mathrm{RHS}}
\end{align*}
where we denote $\bx = (x^1, \dots, x^n)$. Taking the limit as $t^i \to 0$, we notice that the limit of RHS equals $\cL_{X^i}(\log f)(\bx)$. On the other hand, if we consider the following function $g$:
\begin{align*}
g(h) = \E[f(x^1,\dots, X_{h}^i, \dots, x^n) ~\big|~X_0^i=x^i], \quad g(0) = f(\bx)
\end{align*}
We can calculate the limit of LHS as $t^i \to 0$,
\begin{align*}
    \lim_{t^i\to0} \mathrm{LHS} &= \lim_{t^i\to0} \frac{\log(g(t^i)) - \log g(0))}{t^i} = (\log(g))'(0) = \frac{1}{g(0)} g'(0) = \frac{\cL_{X^i}f(\bx)}{f(\bx)}
\end{align*}
Combining them, we have that
\begin{align*}
\frac{\cL_{X^i}f(\bx)}{f(\bx)} \geq \cL(\log f)(\bx)
\end{align*}
Applying this result by choosing $f = p/\beta_{\theta}$, we have that
\begin{align}
\label{eq:jensen_gen}
\dfrac{\cL_{X^i} (p/\beta_\theta)(\bx_{\bt}, \bt)}{(p/\beta_\theta)(\bx_{\bt},\bt)} - \cL_{X^i} \log (p/\beta_\theta)(\bx_{\bt},\bt) \geq 0
\end{align}
This finishes the proof of $\cI_{\mathrm{GESM}} \geq 0$. To see that $\cI_{\mathrm{GESM}}$ is minimized when $\beta_{\theta}(\bx, \bt) \propto p(\bx, \bt) $, note that Jensen's inequality holds the equality sign when the test function is constant. Therefore, \cref{eq:jensen_gen} holds the equality sign for each $\cL_{X^i}$ whenever $p/\beta_{\theta}$ is identically constant, therefore $\cI_{\mathrm{GESM}}$ is optimized when $\beta_{\theta}$ is equivalent to $p$ up to a multiplicative constant.
\end{proof}

\subsection{Proof of Theorem \ref{thm:loss-equiv}}
\begin{proof}
We will start by showing the equivalence between $\cI_{\mathrm{GDSM}}$ and $\cI_{\mathrm{GISM}}$, and then we will demonstrate the equivalence between $\cI_{\mathrm{GISM}}$ and $\cI_{\mathrm{GESM}}$ to finish the proof. We start with the definition of $\cI_{\mathrm{GDSM}}$,
    \begin{align*}
     \cI_{\mathrm{GDSM}} &=
     \underset{\bt, p_0, p_{\bt|0}}{\mathbb{E}} \Bigg[\sum_{i=1}^n \dfrac{\cL_{X^i} (p_{\bt|0}/\beta_\theta)(\bx_{\bt}, \bt)}{(p_{\bt|0}/\beta_\theta)(\bx_{\bt},\bt)} - \cL_{X^i} \log (p_{\bt|0}/\beta_\theta)(\bx_{\bt},\bt) \Bigg] \\
    &=\underset{\bt, p_0}{\mathbb{E}} \int_{\cX} p_{\bt|0}(\bx_{\bt}, \bt) \Bigg[\sum_{i=1}^n \dfrac{{\cL}_{X^i} (\pbeta)(\bx_{\bt}, \bt)}{(\pbeta)(\bx_{\bt},\bt)} - {\cL}_{X^i} \log (\pbeta)(\bx_{\bt},\bt) \Bigg] \rd \bx_{\bt}\\ 
    &=\underset{\bt, p_0}{\mathbb{E}} \int_{\cX}  \Bigg[\sum_{i=1}^n \beta_\theta(\bx_{\bt}, \bt) {\cL}_{X^i} (\pbeta)(\bx_{\bt}, \bt) - p_{\bt|0}(\bx_{\bt}, t) {\cL}_{X^i} \log (\pbeta)(\bx_{\bt},\bt) \Bigg] \rd \bx_{\bt}\\
    \end{align*}
Then, using the properties of the adjoint $\cL^{*}_{X^i}$, we can continue to simplify the objective as,
\begin{align*}
      \cI_{\mathrm{GDSM}} &=\underset{\bt, p_0}{\mathbb{E}} \int_{\cX}  \Bigg[\sum_{i=1}^n {\cL}_{X^i}^* \beta_\theta(\bx_{\bt}, \bt) \cdot p_{\bt|0}(\bx_{\bt}, \bt)/\beta_{\theta}(\bx, \bt) - p_{\bt|0}(\bx_{\bt}, \bt) {\cL}_{X^i} \log (\pbeta)(\bx_{\bt},\bt) \Bigg] \rd \bx_{\bt}\\
    &=\underset{\bt, p_0}{\mathbb{E}} \int_{\cX} p_{\bt|0}(\bx_{\bt}, \bt) \Bigg[\sum_{i=1}^n \dfrac{{\cL}_{X^i}^* 
 \beta_\theta(\bx_{\bt}, \bt)}{ \beta(\bx_{\bt}, \bt)} + {\cL}_{X^i} \log \beta_{\theta}(\bx_{\bt},\bt) - {\cL}_{X^i} \log p_{\bt|0}(\bx_{\bt},\bt) \Bigg] \rd \bx_{\bt}\\
    &=\underset{\bt, p_0}{\mathbb{E}} \E_{p_{\bt|0}} \Bigg[\sum_{i=1}^n \dfrac{{\cL}_{X^i}^* \beta_\theta(\bx_{\bt}, \bt)}{ \beta_{\theta}(\bx_{\bt}, \bt)} + {\cL}_{X^i} \log \beta_{\theta}(\bx_{\bt},\bt) \Bigg] + \mathrm{const} \\
    & = \underset{\bt, p_{\bt}}{\mathbb{E}} \Bigg[\sum_{i=1}^n \dfrac{{\cL}_{X^i}^* \beta_\theta(\bx_{\bt}, \bt)}{ \beta_{\theta}(\bx_{\bt}, \bt)} + {\cL}_{X^i} \log \beta_{\theta}(\bx_{\bt},\bt) \Bigg] + \mathrm{const} \\
    &= \cI_{\mathrm{GISM}} + \mathrm{const}
\end{align*}
This finishes the proof of equivalence between $\cI_{\mathrm{GDSM}}$ and $\cI_{\mathrm{GISM}}$. To show the equivalence between $\cI_{\mathrm{GISM}}$ and $\cI_{\mathrm{GESM}}$, we start with the definition of $\cI_{\mathrm{GESM}}$,
\begin{align*}
    \cI_{\mathrm{GESM}}
        &= \underset{\bt, \bx_{\bt} \sim p(\cdot, \bt)}{\mathbb{E}} \Bigg[\sum_{i=0}^n\dfrac{\cL_{X^i} (p/\beta_\theta)(\bx_{\bt}, \bt)}{(p/\beta_\theta)(\bx_{\bt},\bt)} - \cL_{X^i} \log (p/\beta_\theta)(\bx_{\bt},\bt) \Bigg] \\
        &= \underset{\bt}{\mathbb{E}} \int_{\cX} p(\bx_{\bt}, \bt) \Bigg[\sum_{i=0}^n\dfrac{\cL_{X^i} (p/\beta_\theta)(\bx_{\bt}, \bt)}{(p/\beta_\theta)(\bx_{\bt},\bt)} - \cL_{X^i} \log (p/\beta_\theta)(\bx_{\bt},\bt) \Bigg] \rd \bx_{\bt} \\
        &= \underset{\bt}{\mathbb{E}} \int_{\cX}  \Bigg[\sum_{i=0}^n \beta_\theta(\bx_{\bt}, \bt) \cL_{X^i} (p/\beta_\theta)(\bx_{\bt}, \bt) - p(\bx_{\bt}, \bt)\cL_{X^i} \log (p/\beta_\theta)(\bx_{\bt},\bt) \Bigg] \rd \bx_{\bt}
    \end{align*}
Using again the properties of ajoint $\cL^{*}_{X^i}$, we have
\begin{align*}
    \cI_{\mathrm{GESM}} & = \underset{\bt}{\mathbb{E}} \int_{\cX}  \Bigg[\sum_{i=0}^n \cL_{X^i}^* \beta_\theta(\bx_{\bt}, \bt) (p/\beta_\theta)(\bx_{\bt}, \bt) - p(\bx_{\bt}, \bt)\cL_{X^i} \log (p/\beta_\theta)(\bx_{\bt},\bt) \Bigg] \rd \bx_{\bt}\\
        &= \underset{\bt}{\mathbb{E}} \int_{\cX} p(\bx_{\bt}, \bt) \Bigg[\sum_{i=0}^n \dfrac{\cL_{X^i}^* \beta_\theta(\bx_{\bt}, \bt)}{\beta_\theta(\bx_{\bt}, \bt)} - \cL_{X^i} \log (p/\beta_\theta)(\bx_{\bt},\bt) \Bigg] \rd \bx_{\bt}\\
        & = \underset{\bt, \bx_{\bt} \sim p_{\bt}}{\mathbb{E}} \Bigg[\sum_{i=0}^{n} \dfrac{\cL_{X^i}^* \beta_\theta(\bx_{\bt}, \bt)}{\beta_\theta(\bx_{\bt}, \bt)}  - \cL_{X^i} \log p(\bx_{\bt},\bt) +\cL_{X^i} \log\beta_\theta(\bx_{\bt},\bt) \Bigg] \\
        &= \underset{\bt, \bx_{\bt} \sim p_{\bt}}{\mathbb{E}} \Bigg[\sum_{i=0}^n \dfrac{\cL_{X^i}^* \beta_\theta(\bx_{\bt}, \bt)}{\beta_\theta(\bx_{\bt}, \bt)} + \cL_{X^i} \log\beta_\theta(\bx_{\bt},\bt) \Bigg] + \mathrm{const} \\
        &= \cI_{\mathrm{GISM}} + \mathrm{const}
\end{align*}
This finishes the proof of equivalence of $\cI_{\mathrm{GDSM}}$, $\cI_{\mathrm{GISM}}$, $\cI_{\mathrm{GESM}}$ up to an additive, $\theta$-independent constant.
\end{proof}

\subsection{Proof of Proposition \ref{prop:cont_disc_simp}}
\begin{proof}
In the following, we derive $\cI_{\mathrm{GDSM}}$ when choosing the forward process as in \cref{eq:forward}. Recall that
$\cI_{\mathrm{GDSM}}$ is given as,
\begin{align*}
 \cI_{\mathrm{GDSM}} &=
     \underset{\bt, p_0, p_{\bt|0}}{\mathbb{E}} \Bigg[\underbrace{\dfrac{\cL_{X} (p_{\bt|0}/\beta_\theta)(\bx_{\bt}, \bt)}{(p_{\bt|0}/\beta_\theta)(\bx_{\bt},\bt)} - \cL_{X} \log (p_{\bt|0}/\beta_\theta)(\bx_{\bt},\bt)}_{\mathcal{J}_{X}} + \underbrace{\dfrac{\cL_{Y} (p_{\bt|0}/\beta_\theta)(\bx_{\bt}, \bt)}{(p_{\bt|0}/\beta_\theta)(\bx_{\bt},\bt)} - \cL_{Y} \log (p_{\bt|0}/\beta_\theta)(\bx_{\bt},\bt)}_{\mathcal{J}_{Y}}\Bigg]
\end{align*}
where $\bt = (t, s)$, $\bx_{\bt} = (X_t, Y_s)$, $p_{\bt | 0}(\bx_{\bt}, \bt) = p(X_t, Y_s \mid X_0 = x_0, Y_0 = y_0)$, $\cL_{X}$ and $\cL_{Y}$ are generator of the following dynamics.
\begin{align*}
   \begin{cases}
    \rd X_t = f(X_t, t) \rd t + g(t) \rd B_t\\
    Y_s \sim \operatorname{CTMC}(Q_{s}), \; (X_0, Y_0) \sim p_{\text{data}}(x, y) 
\end{cases}
\end{align*}
See \cref{lem:generator_forward} for a detailed proof. We start by computing the score matching operators $\Phi_{X}$ and $\Phi_{Y}$ for $X$ and $Y$ respectively. Note that $\cL_{X}$ and $\cL_{Y}$ are given by the following expressions for a test function $h = h(x_t, y_s, t, s)$, 
\begin{align*}
    \cL_{X} h = f \cdot \nabla h + \frac{1}{2} g(t)^2 \Delta h, \quad \cL_{Y} h = \sum_{y \in \mathbb{X}} h(x_t, y, t, s)  Q_s(y, y_s) 
\end{align*}
Therefore, for the score matching operator associated with $\cL_{X}$, it can be expressed as 
\begin{align*}
    \Phi_X(h) & = \dfrac{\cL_{X} h}{h} - \cL_{X} \log h \\
    & = \dfrac{f \cdot \nabla h}{h} + \frac{1}{2} g(t)^2 \frac{\Delta h}{h} - f \cdot \nabla \log h - \frac{1}{2} \Delta \log h \\
    & = \frac{1}{2} g(t)^2 \cdot \frac{\nabla \cdot (h \nabla \log h)}{h} - \frac{1}{2}g(t)^2 \Delta \log h \\
    & = \frac{1}{2} g(t)^2  \nabla \log h \cdot \frac{\nabla h}{h} + \frac{1}{2} g(t)^2  \nabla \cdot \nabla \log h - \frac{1}{2} g(t)^2  \Delta \log h 
\end{align*}
Note that $\frac{\nabla h}{h} = \nabla \log h$, $\nabla \cdot \nabla \log h = \Delta \log h$, therefore we derive that
\begin{align*}
    \Phi_{X}(h) = \frac{1}{2}g(t)^2 \| \nabla \log h\|^2
\end{align*}
For the score matching operator associated with $\cL_{Y}$, it can be expressed as,
\begin{align*}
    \Phi_{Y}(h) & = \frac{\cL_{Y} h}{h} - \cL_{Y} \log h \\
& = \sum_{y \in \mathbb{X}} Q_s(y, y_s)  \Big(\dfrac{h(x_t, y, t, s)}{h(x_t, y_s, t,s)} - \log h(x_t, y, t, s)\Big) \\
& = \sum_{y \in \mathbb{X}} Q_s(y, y_s)  \Big(\dfrac{h(x_t, y, t, s)}{h(x_t, y_s, t,s)} - \log \dfrac{h(x_t, y, t, s)}{h(x_t, y_s, t, s)}\Big)
\end{align*}
where the last line follows since $\sum_{y \in \mathbb{X}} Q_s(y, y_s) = 0$. \\

In this case, the $\cL_{X}$ related term $\mathcal{J}_{X}$ in $\cI_{\mathrm{GDSM}}$ can be simplified as,
\begin{align*}
\mathcal{J}_{X} & = \dfrac{\cL_{X} (p_{\bt|0}/\beta_\theta)(\bx_{\bt}, \bt)}{(p_{\bt|0}/\beta_\theta)(\bx_{\bt},\bt)} - \cL_{X} \log (p_{\bt|0}/\beta_\theta)(\bx_{\bt},\bt) \\
& = \underset{\bt, p_0, p_{\bt|0}}{\mathbb{E}} \Bigg[\Phi_{X}\big(\dfrac{p_{\bt|0}}{\beta_{\theta}}\big) \Bigg] \\
& = \underset{\bt, p_0, p_{\bt|0}}{\mathbb{E}} \Bigg[\frac{1}{2} g(t)^2 \bigg\| \nabla \log p(x_t, y_s, t, s \mid x_0, y_0) - \nabla \log \beta_{\theta}(x_t, y_s, t, s) \bigg\|^2\Bigg]
\end{align*}
Moreover, since $(X_t, Y_s)$ are conditionally independent given $(X_0, Y_0)$, we have that
\begin{align*}
\nabla \log p(x_t, y_s, t, s \mid x_0, y_0) & = \nabla \log \big (p(x_t, t \mid x_0, y_0) \cdot p(y_s, s \mid x_0, Y_0)\big) \\
& = \nabla \log p(x_t, t \mid x_0, y_0) + \nabla \log p(y_s, s \mid x_0, y_0) \\
& = \nabla \log p(x_t, t \mid x_0)
\end{align*}
which suggests that under our framework, the multimodal conditional score is identical to the unimodal conditional score. Using this, and set $s^{X}_{\theta}(x_t, y_s, t,s) = \nabla \log \beta_{\theta}(x_t, y_s, t, s)$, we have
\begin{align*}
    \mathcal{J}_{X} = \underset{\substack{t, s, x_0, y_0 \sim p_{0} \\ x_t, y_s \sim p_{t,s | 0}} }{\mathbb{E}}\Bigg[ \frac{1}{2} g^{2}(t) \bigg\| s_{\theta}^{X} - \nabla \log p_{t}(x_t | x_0) \bigg\|^2 \Bigg]
\end{align*}
Similarly, the $\cL_{Y}$ related term $\mathcal{J}_{Y}$ in $\cI_{\mathrm{GDSM}}$ can be simplified as,
\begin{align*}
& \mathcal{J}_{Y} = \underset{\bt, p_0, p_{\bt|0}}{\mathbb{E}} \Bigg[\Phi_{Y}\big(\dfrac{p_{\bt|0}}{\beta_{\theta}}\big) \Bigg] \\
& = \underset{\bt, p_0, p_{\bt|0}}{\mathbb{E}} \Bigg[\sum_{y \in \mathbb{X}} Q_s(y, y_s)  \Big(\dfrac{p(x_t, y, t, s \mid x_0, y_0)}{p(x_t, y_s, t,s \mid x_0, y_0)} \cdot \dfrac{\beta_{\theta}(x_t, y_s, t, s)}{\beta_{\theta}(x_t, y, t,s)} - \log \dfrac{p(x_t, y, t, s \mid x_0, y_0)}{p(x_t, y_s, t,s \mid x_0, y_0)} + \log \dfrac{\beta_{\theta}(x_t, y, t, s)}{\beta_{\theta}(x_t, y_s, t,s)} \Big) \Bigg] \\
& = \underset{\bt, p_0}{\mathbb{E}} \Bigg[ \int_{\mathbb{R}^d} \Bigg(\sum_{y, y_s \in \mathbb{X}} Q_s(y, y_s)  \Big(p(x_t, y, t, s \mid x_0, y_0) \cdot \dfrac{\beta_{\theta}(x_t, y_s, t, s )}{\beta_{\theta}(x_t, y, t,s)} + p(x_t, y_s, t,s \mid x_0, y_0) \cdot \log \dfrac{\beta_{\theta}(x_t, y, t, s)}{\beta_{\theta}(x_t, y_s, t,s)} \Bigg) ~ \rd  x_t\Bigg] \\
& \qquad \qquad \qquad + \mathrm{const}\\
& = \underset{\bt, p_0}{\mathbb{E}} \Bigg[ \sum_{y_s \in \mathbb{X}} \int_{\mathbb{R}^d} \Bigg( \underset{y \sim p(x_t, \cdot, t, s|x_0, y_0)}{\mathbb{E}} Q_s(y, y_s)  \Bigg(\dfrac{\beta_{\theta}(x_t, y_s, t, s)}{\beta_{\theta}(x_t, y, t,s)} + \dfrac{p(x_t, y_s, t,s \mid x_0, y_0)}{p(x_t, y, t,s \mid x_0, y_0)} \cdot \log \dfrac{\beta_{\theta}(x_t, y, t, s)}{\beta_{\theta}(x_t, y_s, t,s)} \Bigg) \Bigg) \rd x_t \Bigg] \\
& \qquad \qquad \qquad + \mathrm{const}
\end{align*}
Now, we exchange variable $y_s$ and $y$, and thus the expression is rewritten to,
\begin{align*}
& \mathcal{J}_{Y} = \underset{\bt, p_0}{\mathbb{E}} \Bigg[ \sum_{y \in \mathbb{X}} \int_{\mathbb{R}^d} \Bigg( \underset{y_s \sim p(x_t, \cdot, t, s|x_0, y_0)}{\mathbb{E}} Q_s(y_s, y)  \Bigg(\dfrac{\beta_{\theta}(x_t, y, t, s)}{\beta_{\theta}(x_t, y_s, t,s)} - \dfrac{p(x_t, y, t,s \mid x_0, y_0)}{p(x_t, y_s, t,s \mid x_0, y_0)} \cdot \log \dfrac{\beta_{\theta}(x_t, y, t, s)}{\beta_{\theta}(x_t, y_s, t,s)} \Bigg) \Bigg) \rd x_t \Bigg] \\
& \qquad \qquad \qquad + \mathrm{const} \\
& = \underset{\bt, p_0, p_{\bt|0}}{\mathbb{E}} \Bigg[ \sum_{y \in \mathbb{X}} Q_s(y_s, y)  \Bigg(\dfrac{\beta_{\theta}(x_t, y, t, s)}{\beta_{\theta}(x_t, y_s, t,s)} - \dfrac{p(x_t, y, t,s \mid x_0, y_0)}{p(x_t, y_s, t,s \mid x_0, y_0)} \cdot \log \dfrac{\beta_{\theta}(x_t, y, t, s)}{\beta_{\theta}(x_t, y_s, t,s)} \Bigg) \Bigg] + \mathrm{const}
\end{align*}
Using again the conditional independence of $(X_t, Y_s)$ given $(X_0, Y_0)$, we have that
\begin{align*}
\dfrac{p(x_t, y, t,s \mid x_0, y_0)}{p(x_t, y_s, t,s \mid x_0, y_0)} &  = \dfrac{p(x_t, t \mid x_0, y_0) p(y, s \mid x_0, y_0)}{p(x_t, t \mid x_0, y_0) p(y_s, s \mid x_0, y_0)} \\
& = \dfrac{p(y, s \mid x_0, y_0)}{p(y_s, s \mid x_0, y_0)} = \dfrac{p(y, s \mid y_0)}{p(y_s, s \mid y_0)}
\end{align*}
which indicates again that the multimodal conditional score is identical to the unimodal conditional score. Set $s_{\theta}^{Y}(x_t, y_s, t, s)_y = \beta_{\theta}(x_t, y, t,s) / \beta_{\theta}(x_t, y_s, t,s)$, we finally arrive at that
\begin{align*}
    \mathcal{J}_{Y} = \underset{\substack{t, s, x_0, y_0 \sim p_{0} \\ x_t, y_s \sim p_{t,s | 0}} }{\mathbb{E}} \Bigg[ \sum_{y \in \mathbb{X}} Q_s(y_s, y)  \Bigg(s_{\theta}^{Y}(x_t, y_s, t,s)_y - \dfrac{p(y, s \mid y_0)}{p(y_s, s \mid y_0)} \cdot \log s_{\theta}^{Y}(x_t, y_s, t,s)_y  \Bigg) \Bigg] + \mathrm{const}
\end{align*}
Note that while the sum in $\mathcal{J}_{Y}$ is over $y \in \mathbb{X}$, in fact when $y = y_s$, the corresponding term is constant and has no contribution to the gradient of $\theta$, therefore we can instead only summing over $y \in \mathbb{X}$ and $y \neq y_s$, recovering the presented expressions in \cref{prop:cont_disc_simp}. This concludes the proof.
\end{proof}

\subsection{Proof of Proposition \ref{prop:backward_text_image}}
We start by considering the function $u(x,y,t,s) = \E[h(X_t, Y_s) ~\big|~ X_0 = x, Y_0 = y]$ as test functions. We start by computing the generator of the forward process. For notational convenience, we use $\cL^X$ and $\cL^Y$ to denote the generators of $X_t$ and $Y_s$, respectively.

\begin{lemma}[Generator of the forward process]
\label{lem:generator_forward}
   Given $(X_t, Y_s)$  following dynamic \eqref{eq:forward}, and a test function $u: \mathbb{R}^d \times \mathbb{X}$, we have that:
   \begin{align}
       &\cL^X u  = \inner{\nabla_x u(x,y)}{f(x,t)} + \frac{1}{2}g^{2}(t) \cdot \inner{\nabla^2_x u(x,y)}{I} = f \cdot \nabla u + \frac{1}{2} g^2(t) \Delta u\\
       &\cL^Y u = \sum_{\yh \in \mathbb{X}} u(x,\yh)  Q_s(\yh, y) = Q_s^{\top} u(x,\cdot)
   \end{align}
\end{lemma}
\begin{proof}
    We start by computing the generator of $Y_s$:
    \begin{align*}
        \cL^Y u &=  \lim_{\Delta s\to 0} \frac{1}{\Delta s} \cdot \E\Big[ u(X_t, Y_{s + \Delta s}) - u(X_t,Y_s) ~\big|~ X_t = x, Y_s = y\Big] \\
        &=\lim_{\Delta s\to 0} \frac{1}{\Delta s} \cdot \sum_{\yh \in \mathbb{X}} u(x,\yh) \big(\P(Y_{s+\Delta s} = \yh ~\big|~ Y_s = y) - \P(Y_s = \yh ~\big|~ Y_s = y)\big) \\
        &= \sum_{\yh \in \mathbb{X}} u(x,\yh)\lim_{\Delta s \to 0} \frac{1}{\Delta s} \big(\P(Y_{s+\Delta s} = \yh ~\big|~ Y_s = y) - \P(Y_s = \yh ~\big|~ Y_s = y)\big) \\
        & = \sum_{\yh \in \mathbb{X}} u(x,\yh)  Q_s(\yh, y) = Q_s^{\top} u(x,\cdot)       
    \end{align*}
    Similarly, for the generator of $X_t$,
    \begin{align*}
        \cL^{X}u  & = \lim_{\Delta t \to 0} \frac{1}{\Delta t} \E\Big[u(X_{t + \Delta t}, Y_s) - u(X_t, Y_s) ~\big|~ X_t = x, Y_s = y\Big] \\
        &= \lim_{\Delta t \to 0} \frac{1}{\Delta t} \cdot \E \Big[ \int_t^{t+\Delta t} \inner{f(X_\tau, Y_s)}{\nabla_x u(X_\tau, ,Y_s)} + \frac{1}{2} \inner{g^2(\tau) I}{\nabla_{x}^2 u(X_\tau, Y_s)} \rd \tau ~\big|~ X_t = x, Y_s = y\Big] \\
        &= \E \Big[ \inner{f(X_t, Y_s)}{\nabla_x u(X_t, ,Y_s)} + \frac{1}{2} \inner{g^2(t)}{\nabla^2 u(X_t, Y_s)} ~\big|~ X_t = x, Y_s = y\Big] \\
        & = (f \cdot \nabla u)(x,y) + \frac{1}{2} g^2(t) \Delta u(x,y)
    \end{align*}
    where we uses Ito's lemma and the fact that the martingale vanishes under the expectation.
\end{proof}

We now compute the generator for the backward process:
\begin{equation}
\label{eq:time-reversal-backward}
   \begin{cases}
    \begin{aligned}
        \rd \overleftarrow{X}_t &= -f(\overleftarrow{X}_t, T-t) + g^2(T-t) \nabla_{x} \log \overleftarrow{p}_{t,s}(\overleftarrow{X}_t, \overleftarrow{Y}_s) \rd t + g(t) \rd \overleftarrow{W}_t
    \end{aligned} \\
    \overleftarrow{Y}_s \sim \operatorname{CTMC}\big(\overleftarrow{Q}(\overleftarrow{X}_t, t,s)\big) \\
    (\overleftarrow{X}_0, \overleftarrow{Y}_0) \sim p(x, y, T, T)
\end{cases}
\end{equation}
where $\overleftarrow{p}_{t,s} = p_{T-t, T-s}$, $\overleftarrow{Q}(X_t, t, s)$ is a rate matrix with $y', y$ entry being $\frac{\overleftarrow{p}_{t,s}(X_t, y')}{\overleftarrow{p}_{t,s}(X_t, y)}(Q_{T-s})_{yy'}$. We denote $\cL^{X}_{\leftarrow}$ and $\cL_{\leftarrow}^{Y}$ as the generator of backward process $\overleftarrow{X}_t$, $\overleftarrow{Y}_s$ in \cref{eq:time-reversal-backward} respectively. 
\begin{lemma}[Generator of the backward process]
\label{lem:generator_backward}
   Given $X_t, Y_s$ following \eqref{eq:time-reversal-backward} and a test function $u: \mathbb{R}^d \times \mathbb{X}$, we have that:
    \begin{align*}
        \cL^{X}_{\leftarrow} u & = -\inner{\nabla_x u(x,y)}{f(x,T-t)} + \frac{1}{2} g^2(T-t) \inner{\nabla^2_x u(x,y)}{I} - \inner{\nabla_x u(x,y)}{g^2(T-t) \nabla_x \log \overleftarrow{p}_{t,s}(x,y)} \\
        \cL^{Y}_{\leftarrow} u & = \overleftarrow{Q}(x, t, s)^{\top} u = \sum_{\yh \in \mathbb{X}} u(x,\yh) \cdot \overleftarrow{Q}(x, t, s)(\yh,y) 
    \end{align*}
\end{lemma}
The proof is highly similar to that of \cref{lem:generator_backward}, and thus we omit it here to avoid being repetitive. With $\cL^X$, $\cL^Y$, $\cL^{X}_{\leftarrow}$ and $\cL^{Y}_{\leftarrow}$ being computed, we can now derive the Fokker Planck equation for both the forward process $(X_t, Y_s)$ and backward process $(\overleftarrow{X}_t, \overleftarrow{Y}_s)$, as is given in \cref{lem:fokker_forward} and \cref{lem:fokker_backward}.
\begin{lemma}[Fokker Planck Equation of Forward Process]
\label{lem:fokker_forward}
   For $X_t, Y_s$ following \eqref{eq:forward}, let $p(x,y,t,s)$ denotes the density of $X_t, Y_s)$ at $X_t = x, Y_s = y$, we have that:
   \begin{align}
   \label{eq:fokker_forward}
       \partial_t p(x,y,t,s) = \cL^{X,*} p(x,y,t,s), \quad \partial_s p(x,y,t,s) = \cL^{Y,*} p(x,y,t,s)  
   \end{align}
   Where $\cL^{X,*}, \cL^{Y,*}$ represent the adjoint of $\cL^X, \cL^Y$ and: 
   \begin{align*}
       &\cL^{X,*} p(x,y,t,s) = - \nabla \cdot (f(x,t) p(x,y,t,s)) + \frac{1}{2} g^2(t) \Delta p(x,y,t,s)\\
       &\cL^{Y,*} p(x,\cdot,t,s) =  Q_s p(x,\cdot,t,s) = \sum_{\yh \in \mathbb{X}} Q_s(y, \yh) \cdot p(x, \yh, t, s)
   \end{align*}
\end{lemma}

\begin{lemma}[Fokker Planck Equation of Backward Process]
\label{lem:fokker_backward}
   For $\overleftarrow{X}_t, \overleftarrow{Y}_s$ following \eqref{eq:forward}, let $\overline{p}(x,y,t,s)$ denotes the density of $\overleftarrow{X}_t, \overleftarrow{Y}_s)$ at $\overleftarrow{X}_t = x, \overleftarrow{Y}_s = y$, we have that:
   \begin{align}
   \label{eq:fokker_backward}
       \partial_t \overline p(x,y,t,s) = \cL^{X,*}_\leftarrow \overline p(x,y,t,s), \qquad \partial_s \overline p(x,y,t,s) = \cL^{Y,*}_\leftarrow \overline p(x,y,t,s)  
   \end{align}
   where $\cL^{X,*}_\leftarrow, \cL^{Y,*}_\leftarrow$ represent the adjoint of $\cL^X_\leftarrow, \cL^Y_\leftarrow$ and: 
    \begin{align*}
        &\cL^{X,*}_\leftarrow \overline p(x,y,t,s) =  \nabla \cdot (f(x,T-t) \overline{p}(x,y,t,s)) + \frac{1}{2} g^2(T-t) \Delta \overline{p}(x,y,t,s) -g^2(T-t) \nabla \cdot \big(\overline p(x,y,t,s) \nabla_x \log \overleftarrow{p}_{t,s}(x,y) \big) \notag\\
        & \cL^{Y,*}_\leftarrow \overline p(x,\cdot,t,s) = \overleftarrow{Q}(x, t, s) \overline{p}(x, \cdot, t, s) = \sum_{\yh \in \mathbb{X}}  \overleftarrow{Q}(x, t, s)(y, \yh) \overline{p}(x, \yh, t, s)
    \end{align*}
\end{lemma}
The computation central to the proof of \cref{lem:fokker_forward} and \cref{lem:fokker_backward} is the calculation of the adjoint operator, which can be done with standard techniques, such as integration by parts. The derivation of the Fokker Planck equations directly follows from definitions. Therefore, we also omit the proof here. With these results, we are now ready to show that $(X_t, Y_s)$ and $\overleftarrow{X}_t, \overleftarrow{Y}_s$ are time reversals of each other.
\begin{lemma}[Time Reversal] $\overline{p}(x,y, t,s) = \overleftarrow{p}_{t, s}(x,y) = p(x, y, T -t, T-s)$, where $p$ is the solution to the Fokker Planck equation of the forward process in \cref{eq:fokker_forward}, $\overline{p}$ is the solution to the Fokker Planck equation of the backward process in \cref{eq:fokker_backward}.
\end{lemma}
\begin{proof}
We will prove the result by showing that $\overline{p}(x,y, t,s) = p(x,y, T-t, T-s)$ satisfies the Fokker Planck equations given in \cref{eq:fokker_backward}. We start by showing the $\overleftarrow{X}_t$ related equation. Substituting in $p(x,y,T-t,T-s)$, we have that
\begin{align*}
    \partial_t \overline{p}(x,y,t,s) = \partial_t(p(x,y,T-t, T-s)) = - \partial_tp(x,y,T-t,T-s) =  - \cL^{X,*}p(x,y,T-t,T-s)
\end{align*}
where the last equality holds due to \cref{eq:fokker_forward}. On the right side of the equation, the expression can be simplified to:
\begin{align*}
 \cL^{X,*}_\leftarrow \overline{p}(x,y,t,s) &= \cL^{X,*}_\leftarrow p(x,y,T-t,T-s) \\
 & = \nabla \cdot (f(x,T-t) p(x,y,T-t, T-s) + \frac{1}{2}g^2(T-t) \Delta p(x,y,T-t,T-s) \\
 & \qquad \qquad - g^2(T-t) \nabla \cdot (p(x,y,T-t,T-s) \nabla \log p(x,y,T-t,T-s)) \\
 & = \nabla \cdot (f(x,T-t) p(x,y,T-t, T-s)  + \frac{1}{2}g^2(T-t) \Delta p(x,y,T-t,T-s) \\
 & \qquad \qquad - g^2(T-t) \nabla \cdot \nabla p(x,y,T-t, T-s) \\
 & =  \nabla \cdot (f(x,T-t) p(x,y,T-t, T-s) - \frac{1}{2}g^2(T-t) \Delta p(x,y,T-t,T-s) \\
 & = - \cL^{X,*}p(x,y,T-t,T-s)
\end{align*}
Therefore, we show that $ \partial_t \overline p(x,y,t,s) = \cL^{X,*}_\leftarrow \overline p(x,y,t,s)$ when $\overline{p}(x,y, t,s) = p(x,y, T-t, T-s)$. For the $\overleftarrow{Y}_s$ related equation, we have that
\begin{align*}
    \partial_s \overline{p}(x,y,t,s) = \partial_s(p(x,y,T-t,T-s)) = - \partial_s p(x,y,T-t,T-s) = -\cL^{Y,*}p(x,y,T-t,T-s)
\end{align*}
Similarly, on the right size of the equation, we have,
\begin{align*}
\cL^{Y,*}_{\leftarrow} \overline{p}(x,y,t,s) & = \cL^{Y,*}_{\leftarrow} p(x,y,T-t,T-s) \\
& = \sum_{\yh \in \mathbb{X}} \overleftarrow{Q}(x,t,s)(y, \yh) p(x,\yh,T-t,T-s) \\
& = \sum_{\yh \in \mathbb{X}, \yh \neq y} \overleftarrow{Q}(x,t,s)(y, \yh) p(x,\yh,T-t,T-s) + \overleftarrow{Q}(x,t,s)(y,y) p(x,y,T-t, T-s) \\
& = \sum_{\yh \in \mathbb{X}, \yh \neq y}\overleftarrow{Q}(x,t,s)(y, \yh) p(x,\yh,T-t,T-s) - \sum_{\yh \in \mathbb{X}, \yh \neq y} \overleftarrow{Q}(x,t,s)(\yh,y) p(x,y,T-t, T-s)
\end{align*}
where in the last step, we use that $\overleftarrow{Q}(x,t,s)(y,y) = -\sum_{\yh \in \mathbb{X}, \yh \neq y} \overleftarrow{Q}(x,t,s)(\yh,y)$. We perform such a simplification since we can only relate $\overleftarrow{Q}(x,t,s)$ and $Q_{T-s}$ on non-diagonal entries. Then, it holds that 
\begin{align*}
\cL^{Y,*}_{\leftarrow} \overline{p}(x,y,t,s) & = \sum_{\yh \in \mathbb{X}, \yh \neq y} Q_{T-s}(\yh, y) \cdot \dfrac{p(x, y, T-t, T-s)}{p(x,\yh, T-t, T-s)} p(x,\yh,T-t,T-s) \\
& \qquad \qquad -  Q_{T-s}(y, \yh) \cdot \dfrac{p(x, \yh, T-t, T-s)}{p(x,y, T-t, T-s)} p(x,y,T-t,T-s)  \\
& \sum_{\yh \in \mathbb{X}, \yh \neq y} Q_{T-s}(\yh, y) p(x,y,T-t, T-s) - Q_{T-s}(y,\yh) p(x,\yh,T-t,T-s) \\
& = \sum_{\yh \in \mathbb{X}} Q_{T-s}(\yh, y) p(x,y,T-t, T-s) - Q_{T-s}(y,\yh) p(x,\yh,T-t,T-s) \\
& = p(x,y,T-t, T-s) \cdot \big(\sum_{\yh \in \mathbb{X}} Q_{T-s}(\yh, y)\big) - \sum_{\yh \in \mathbb{X}} Q_{T-s}(y,\yh) p(x,\yh,T-t,T-s) \\
& = - \sum_{\yh \in \mathbb{X}} Q_{T-s}(y,\yh) p(x,\yh,T-t,T-s) \\
& = -\cL^{Y,*}p(x,y,T-t,T-s)
\end{align*}
where in the derivation, we use the definition of $\overleftarrow{Q}(x,t,s)$ and the fact that $\sum_{\yh \in \mathbb{X}}Q_{T-s}(\yh, y) = 0$. Therefore, we have also shown that $ \partial_t \overline p(x,y,t,s) = \cL^{Y,*}_\leftarrow \overline p(x,y,t,s)$ when $\overline{p}(x,y, t,s) = p(x,y, T-t, T-s)$. Together with the fact that the initial conditions are matched by construction, i.e.,
\begin{align*}
    \overline{p}(x,y,0,0) = p(x,y,T,T)
\end{align*}
we conclude the proof of the time-reversal argument, as well as \cref{prop:backward_text_image}.
\end{proof}

%% file: cont_disc.tex
\section{Design choice for Continuous-Discrete Multimodal Diffusion Models}
\label{app:cont_disc}

In this section, we go deeper into the design choices of the newly proposed continuous-discrete multimodal diffusion models, including the choice for forward process, score parameterization, and sampling algorithms for such diffusion models.
\subsection{Choice of Forward Process}
\label{app:cont_disc_forward}
We consider the following specific choice of forward process for the Continuous-Discrete Diffusion Model, where we choose $X_t$ to be subjected to a time-rescaled Ornstein–Uhlenbeck process \citep{song2020score} and $Y_s$ to be subjected to a masked discrete diffusion model \citep{ou2024your, sahoo2024simple, shi2024simplified}. Both choices are effective when modeling unimodal distributions, prompting us to combine them for the design of multimodal diffusion models on their product space. 
\begin{equation}
\label{eq:cont_disc_forward}
   \begin{cases}
    \rd X_t = -\beta_t X_t \rd t + \sqrt{2\beta_t} \rd B_t\\
    Y_s \sim \operatorname{CTMC}(\sigma_s Q^{\mathrm{mask}}) \\
    (X_0, Y_0) \sim p_{\text{data}}(x, y) 
\end{cases}
\end{equation}
To define masked discrete diffusion models, we need to introduce an additional mask token $\mathbf{M}$ into the state space $\mathbb{X}$. Therefore, the transition rate matrix $Q^{\mathrm{mask}}$ is a transition matrix defined on the extended state space $\tilde{\mathbb{X}} = \mathbb{X} \cup \{\mathbf{M}\}$, given by
\[Q^{\mathrm{mask}} = \begin{pmatrix}
    -1 & 0 &\dots &0 \\
    0 & -1 & \dots & 0 \\
    \vdots & \vdots & \ddots & \vdots \\
    1 & 1 \dots & 1 & 0
\end{pmatrix}\] 
where the last row corresponds to $\mathbf{M}$. \\

Consider now the practical case, where $X_t \in \mathbb{R}^{d}$, and $Y_s$ consists of a sequence of tokens, i.e., $Y_s \in \mathbb{X}^{n}$. We represent $Y_s$ as $Y_s = y_1 \dots y_i \dots y_n$. 

For pure masked discrete diffusion, it has been shown that its score has a special factorization \citep{ou2024your}. This enables users to parameterize only the probability of the clean data distribution conditioned on unmasked tokens, thereby representing the score. In the following, we demonstrate that this special structure is inherited in the continuous-discrete diffusion model when the discrete part is a masked discrete diffusion, which enables a more effective score parameterization, as well as a variance-reduced version of the training objective. 

Now, consider $\boldsymbol{y} = y_1 \dots y_i \dots y_n$, and $\hat{\boldsymbol{y}} = y_1 \dots \hat{y}_i \dots y_n$, where $\boldsymbol{y}$ differs from $\hat{\boldsymbol{y}}$ only at the $i$-th position, $\hat{y}_i = \mathbf{M}$ while $y_i \neq \mathbf{M}$
\begin{proposition}
\label{prop:cont_disc_special}
Let $p(x, \boldsymbol{y}, t, s)$ be the density of \cref{eq:cont_disc_forward} at $X_t = x, Y_s = \boldsymbol{y}_s$, then the discrete score has the following form
\begin{align}
\label{eq:score_factorize}
    \dfrac{p(X_t = x_t, Y_s = \boldsymbol{y}, t, s)}{p(X_t = x_t, Y_s = \hat{\boldsymbol{y}}, t, s)} = \dfrac{e^{-\overline{\sigma}_s}}{1-e^{-\overline{\sigma}_s}} \mathbb{P}(Y_0^i = y_i \mid \boldsymbol{y}^{\mathrm{UM}}, X_t, t)
\end{align}
where $\overline{\sigma}_s = \int_{0}^{s} \sigma(\tau) \rd \tau$, $\boldsymbol{y}^{\mathrm{UM}}$ contains the unmasked tokens of $\boldsymbol{y}$.
\end{proposition}
\begin{proof}
    We begin by noting that, given the forward process defined in \cref{eq:cont_disc_forward}. We can solve the forward process analytically as follows:
    \begin{align}\label{eq:sol_ctmc}
        \P(Y^i_s = y \mid A) &= 
        \begin{cases}
            e^{-\overline{\sigma}_s} \cdot \P(Y^i_0 = y \mid A), & y \neq \mathbf{M} \\
            1 - e^{-\overline{\sigma}_s}, & y=\mathbf{M}
        \end{cases}
    \end{align}
    For any event $A$. We will use this fact several times in our proof. This is true since when $y \neq \mathbf{M}$, 
    \begin{align*}
    \P(Y_s^i  = y \mid A) = \P(Y_s^i = y \mid Y_0^i = y, A) \cdot \P(Y_0^i = y \mid A)
    \end{align*}
    Our proof consists of two steps. First, we demonstrate that the diffusion process can be factored into a conditional probability and a time-dependent term. Secondly, we demonstrate that we can further simplify this conditional probability to contain only probabilities in terms of the clean data distribution. 
    
    \textbf{Step 1:} The discrete score in \cref{eq:score_factorize} is given by:
    \begin{align*}
        \frac{
        \P(X_t = x_t, Y_s^{1} = y_1, \dots, Y_s^i = y_i, \dots, Y_s^n = y_n)
        }{
        \P(X_t = x_t, Y_s^{1} = y_1, \dots, Y_s^i = \mathbf{M}, \dots, Y_s^n = y_n)
        }
    \end{align*}
    To simplify the notation, we define the following notation $A_i = \{Y_s^{k} = y_k : k \neq i\} \cap \{X_t = x_t\}$. Then, using Bayes' rule, we can rewrite the discrete score as:
    \begin{align*}
        \dfrac{
        \P(Y_s^i = y_i \mid A_i) \P(A_i)
        }{\P(Y_s^i = \mathbf{M} \mid A_i) \P(A_i)}=  
        \dfrac{
        \P(Y_s^i = y_i \mid A_i)
        }{\P(Y_s^i = \mathbf{M} \mid A_i)}  = \dfrac{e^{-\overline{\sigma}_s}}{1-e^{-\overline{\sigma}_s}} \cdot \P(Y_0^i = y_i \mid A_i)
    \end{align*}
    \textbf{Step 2:} We now show that we can simplify $A_i$ by removing the conditioning on tokens that are masked. But firstly we do a simple calculation that will come in handy, given events $A,B,C$ one has that:
    \begin{align*}
        \P(A \mid B\cap C) &= \frac{\P(A\cap B \cap C)}{\P(B\cap C)} 
        = \frac{\P (A\cap B | C) \P(C) }{\P(B\cap C)}  
        = \frac{\P(C | A\cap B) \P(A \cap B) \P(C)}{\P(C) \P(B\cap C)} \\
        &= \frac{\P(A \cap B) \P(C | A \cap B)}{\P(B)) \P(C|B)} = \frac{\P(A | B) \P(C | A \cap B)}{\P(C|B)} \\
    \end{align*}
    Now, assume that $l$ is a position such that $y_l = \mathbf{M}$. We denote $A_i^l = \{Y_s^{k} = y_k : k \neq i,l\} \cap \{X_t = x_t\}$ the event given by the remaining tokens. We can then use the calculation above to write:
    \begin{align*}
        \P(Y_0^i = y_i \mid A_i) &= \P(Y_0^i = y_i \mid A_i^l, Y^l_s = M) \\
        &= \dfrac{\P(\{Y_0^i = y_i\} \cap A_i^l \cap \{ Y^l_s = M\})}{\P(A_i^l \cap \{Y^l_s = M\}}) \\
        &= \P(Y_0^i = y_i \mid A_i^l) \cdot \dfrac{\P(Y_s^l = \mathbf{M} \mid \{Y_0^i = y_i\} \cap A_i^l)}{\P(Y_s^l = \mathbf{M} \mid A_i^l)} \\
        &= \P(Y_0^i = y_i \mid A_i^l) 
    \end{align*}
    where in the last step, we used \cref{eq:sol_ctmc}. The computation above shows that we can remove any events relevant to masked tokens, and repeating this process yields the result that the event being conditioned on only contains the unmasked part. This finishes the proof.
\end{proof}
With \cref{prop:cont_disc_special}, the score entropy in the training objective described in \cref{prop:cont_disc_simp} can be further simplified to cross-entropy loss, as is widely discussed in the literature of masked discrete diffusion model training \citep{ou2024your, sahoo2024simple, shi2024simplified}. \cref{prop:cont_disc_special} also implies that when designing the score network backbone, the discrete score does not need the input of time $s$. However, it's still dependent on the time $t$ of the continuous modality. This is a notable difference from the pure masked discrete diffusion models considered in the literature \citep{ou2024your, nie2024scaling}.

\subsection{Noisy Guidance}
In \cref{alg:noisy_guidance_cont}, we present the detailed algorithm for using Noisy guidance for continuous score. In \cref{alg:noisy_guidance_disc}, we present the detailed algorithm for using Noisy guidance for discrete score, whose computation is not mentioned in the main text. Note that in \cref{alg:noisy_guidance_disc}, instead of directly doing geometric average of $s^{\text{uncond}, y}_\theta$ and $s^{\text{cond},y}_\theta$ as is suggested in \citet{nisonoff2024unlocking}, we compute a arthimetic average of the corresponding logits, then use softmax to calculate the actual guided score. Such a practice is considered in \citet{chang2022maskgit}, and recently it has been shown to have theoretical advantages. 
\begin{algorithm}[H]
\caption{Noisy Guidance for continuous score}
\label{alg:noisy_guidance_cont}
\begin{algorithmic}[1]
\Require $x_t, t$ : noisy image with noise level, model: $s_\theta(x_t,y_s,t,s,\omega)$, $y_0$ : clean text, $\omega$ : Guidance Strength, $\sigma$ : Conditioning Noise Level
\Ensure $s_\theta^x$ : Guided continuous score
\State $y^{\text{noisy}}_{\sigma} \sim p_{\sigma |0}(y|y_0)$
\State $s^{\text{uncond}, x}_\theta \gets s_{\theta}(x, y^{\text{noisy}}_\sigma, t, \sigma,\omega_t)$, $s^{\text{cond}, x}_\theta \gets s_{\theta}(x_t, y_0, t, 0,\omega_t)$
\State $s_\theta^{x} \gets \omega \cdot s^{\text{cond},x}_\theta + (1-\omega) \cdot s^{\text{uncond},x}_\theta$
\State \Return  $s^x_\theta$
\end{algorithmic}
\end{algorithm}
\begin{algorithm}[H]
\caption{Noisy Guidance for discrete score}
\label{alg:noisy_guidance_disc}
\begin{algorithmic}[1]
\Require $y_s, s$ : noisy text with noise level, model: $s_\theta(x_t,y_s,t,s,\omega)$, $x_0$ : clean image, $\omega$ : Guidance Strength, $\sigma$ : Conditioning Noise Level
\Ensure $s_\theta^y$ : Guided discrete score
\State $x^{\text{noisy}}_{\sigma} \sim p_{\sigma|0}(x|x_0)$
\State $s^{\text{uncond},y}_\theta \gets s_{\theta}(x^{\text{noisy}}_\sigma, y_s, \sigma, s,\omega_s)$, $s^{\text{cond},y}_\theta \gets s_{\theta}(x_0, y_s, 0, s, \omega_s)$
\State $s^{\text{uncond},y}_\theta = \operatorname{softmax}(\ell^{\mathrm{uncond}})$, $s^{\text{cond},y}_\theta = \operatorname{softmax}(\ell^{\mathrm{cond}})$
\State $s_\theta^{y} \gets \operatorname{softmax}\big(\omega \cdot \ell^{\mathrm{cond}} + (1-\omega) \cdot \ell^{\mathrm{uncond}}\big)$
\State \Return  $s_\theta^y$
\end{algorithmic}
\end{algorithm}

\subsection{Samplers}
For inference of continuous-discrete multimodal diffusion models, we consider the following samplers. For conditional generation of discrete or continuous modality, see \cref{alg:discrete_sampler} and \cref{alg:continuous_sampler}. Among all the pseudo code, we set times $= (N-i)/N$, $i = 0,1, \dots, N$ as the inference times and $\rd t = -1/N$ as the time steps, where $N$ is the number of total inference/discretization steps. We use $\tau$-leaping for discrete modality sampling and Heun's method for continuous modality sampling. All the inference algorithms are written with the presence of guidance and guidance intervals. 

\begin{minipage}[t]{0.48 \linewidth}%
\begin{algorithm}[H]
\caption{Discrete Sampler with $\tau$-leaping}
\label{alg:discrete_sampler}
\begin{algorithmic}[1]
\Require $N$ : Number of steps, $\omega$ : Guidance Strength\\
$[a,b]$ : Guidance  Interval, model : $s_\theta(x_t,y_s,t,s,\omega)$ \\
$x$ : a clean image condition
\Ensure $y_0 \sim p_{\text{data}}(\cdot |x)$
\State $y_t \gets [\mathbf{M}, \dots, \mathbf{M}]$ 
\For{$t$ in times}
    \State $\omega_t = \omega \text{ if } t\in[a,b] \text{ else } 1.$ 
    \State $s^x_\theta, s^y_\theta \gets s_{\theta}(x, y_t, 0, t,\omega_t)$ 
    \State $y_t \gets \tau\text{-leaping}(s^y_\theta, y_t, t, |\rd t|)$
\EndFor
    \State \Return $y_0$

\end{algorithmic}
\end{algorithm}

\end{minipage}%
\hfill
\begin{minipage}[t]{0.48 \linewidth}%
\begin{algorithm}[H]
\caption{Continuous Sampler with Heun's method}
\label{alg:continuous_sampler}
\begin{algorithmic}[1]
\Require $N$ : Number of steps, $\omega$ : Guidance Strength \\
$[a,b]$ : Guidance  Interval, model : $s_\theta(x_t,y_s,t,s,\omega)$ \\
$y$ : A clean text condition 
\Ensure $x_0 \sim p_{\text{data}}(\cdot |y)$
\State $x_t \gets \mathcal{N}(0,I)$ 
\For{$t$ in times}
    \State $\omega_t = \omega \text{ if } t\in[a,b] \text{ else } 1.$ 
    \State $s^x_\theta, s^y_\theta \gets s_{\theta}(x_t, y, t, 0,\omega_t)$ 
    \State $v_{\text{old}} = f(x_t,t)- \frac{1}{2}  g^2(t) s^x_\theta$  
    \State $\hat{x} \gets x_t + v_{\text{old}}\rd t$, \quad $\hat{s}^x_\theta, \hat{s}^y_\theta \gets s_{\theta}(\hat{x}, y, t + \rd t, 0,\omega_t)$ 
    \State $v_{\text{new}} = f(\hat{x}_t,t)- \frac{1}{2}  g^2(t)\hat{s}^x_\theta$  
    \State $x_t \gets x_t + \frac{1}{2}\cdot (v_{\text{old}} + v_{\text{new}})\rd t$
\EndFor
    \State \Return $x_0$
\end{algorithmic}
\end{algorithm}
\end{minipage}%

We describe one multimodal sampler for the joint generation with a continuous-discrete multimodal diffusion model in \cref{alg:multmodal_sampler}. Essentially, \cref{alg:multmodal_sampler} combines $\tau$-leaping for discrete modality and Heun's method for continuous modality, each depicted in \cref{alg:discrete_sampler} and \cref{alg:continuous_sampler}. However, these choices are selected without being heavily optimized to tailor to this case, and potentially, there exist much more effective and efficient samplers. For example, note that Heun's method is a second-order ODE sampler, while $\tau$-leaping is usually considered to be a first-order CTMC sampling algorithm \citep{ren2024discrete}. This means that the discrete score obtained at the mid-point $\hat x$ during the inference step of Heun's method is not being used, which causes a waste of computation. A potential way to improve is to replace the first-order discrete sampler with a second-order variant, such as the $\theta$-Trap algorithm introduced in \citet{ren2025fast}. We leave this direction for future investigation.

\begin{algorithm}[H]
\caption{Multimodal Sampler with $\tau$-leaping and Heun's Method}
\label{alg:multmodal_sampler}
\begin{algorithmic}[1]
\Require $N$ : Number of steps, $\omega$ : Guidance Strength, $[a,b]$ : Guidance  Interval, model : $s_\theta(x_t,y_s,t,s,\omega)$ 
\Ensure $x_0,y_0 \sim p_{\text{data}}$
\State $x_t \gets \mathcal{N}(0,I)$, $y_t \gets[\mathbf{M}, \dots, \mathbf{M}]$ 
\For{$t$ in times}
    \State $\omega_t = \omega \text{ if } t\in[a,b] \text{ else } 1.$ 
    \State $s^x_\theta, s^y_\theta \gets s_{\theta}(x_t, y_t, t, t,\omega_t)$, \quad $v_{\text{old}} = f(x_t,t)- \frac{1}{2}  g^2(t) s^x_\theta$  
    \State $\hat{x} \gets x_t + v_{\text{old}}\rd t$
    \State $\hat{s}^x_\theta, \hat{s}^y_\theta \gets s_{\theta}(\hat{x}, y_t, t+\rd t, t,\omega_t)$, \quad $v_{\text{new}} = f(\hat{x}_t,t)- \frac{1}{2}  g^2(t)\hat{s}^x_\theta$  
    \State $x_t \gets x_t + \frac{1}{2}\cdot (v_{\text{old}} + v_{\text{new}})\rd t$, \quad  $y_t \gets \tau\text{-leaping}(s^y_\theta, y_t, t, |\rd t|)$
\EndFor
    \State \Return $x_0, y_0$

\end{algorithmic}
\end{algorithm}

%% file: diff_details.tex
\section{Experimental Details on Text-Image Generation}
\label{app:text_image}
\subsection{Choice of Forward Process}
We consider the same forward process discussed in \cref{app:cont_disc_forward}, with $\beta_t$ and $\sigma_s$ given as
\begin{align*}
& \beta_t = 500 \cdot (\sqrt{\beta_{\mathrm{start}}} (1-t) + t \sqrt{\beta_{\mathrm{end}}})^2 \\
& \sigma_s = \frac{1-\delta}{1-(1-\delta)s} \\
& \beta_{\mathrm{start}} = 0.00085, \quad \beta_{\text{end}}=0.0120, \quad \delta = 10^{-5}
\end{align*}

\subsection{Training Strategy}
We divide our training into several stages. This is a standard practice for training vision-language models. In text-to-image diffusion models, a pretrained text encoder is used to achieve alignment between the text semantics and the image features. Popular choices in the literature are using CLIP or T5 as text encoders \cite{esser2024scaling}. However, in our use case, we require training on masked text, for which the availability of pretrained encoders is limited. For this reason, we decided not to use a pretrained text encoder. This has the advantage that we don't rely on any pretraining, which reduces the computational requirements of our model.

\textbf{Stage 1: Text-image Alignment}  During this stage, we train both the joint embedding and the continuous decoder. We allow noisy text to be received as input to our model, meaning that we train on all possible combinations of $ s$ and $ t$, but without worrying about the text prediction task. We present all training hyperparameters across all stages below.

\textbf{Stage 2: Text prediction and Image Improvement} In this second stage, we freeze the joint embedding. We found that by doing so, we can simplify the training. The joint embedding is now capable of generating meaningful latent representations, which can then be used to predict clean text from masked tokens and a latent image representation. 

\textbf{Stage 3: Multimodal Generation} Finally, we train both the image and text decoders. This is useful because the image decoder hasn't been trained specifically to predict from the frozen joint embeddings. Training the text decoder is not necessary, but we can get some extra training time by doing so. After this stage, our model is now capable of performing all tasks. 

\textbf{Optional - Stage 4: Fine-tuning on downstream Tasks} When necessary, our models can be fine-tuned on downstream tasks to improve the performance.

\subsection{Sampling}
For sampling, we use the samplers described in \cref{alg:discrete_sampler}, \cref{alg:continuous_sampler}, and \cref{alg:multmodal_sampler}, where we do not use guidance for the discrete component. Our default values for sampling are presented in Table \ref{tab:sampling_hyperparams}. 
\begin{table}[!t]
    \centering
    \renewcommand{\arraystretch}{1.2}
    \centering
    \caption{Hyperparameters for inference of different tasks}
    \label{tab:sampling_hyperparams}
    \begin{tabular}{llll}
        \toprule
        \textbf{Parameter} & \textbf{text to image} & \textbf{image to text} & \textbf{joint} \\
        \midrule
         Number of Steps & $50$ & $50$ & $50$ \\
         Guidance Scale & $5.0$ & $1.0$ & $5.0$ \\
         Guidance Interval & $[0.3,0.8]$ & - & $[0.3,0.8]$\\
         Condition Noise Level & $0.77$ & - & $1.0$ \\
         Early Stopping & $10^{-5}$ & $10^{-5}$ & $10^{-5}$ \\
        \bottomrule
    \end{tabular}
    \vspace{-1em}
\end{table}

\begin{table}[!t]
    \centering
    \renewcommand{\arraystretch}{1.2}
    
    \begin{minipage}[t]{0.45\linewidth}%
        \centering
        \caption{Model hyperparameters}
        \label{tab:model_hyperparams}
        \begin{tabular}[b]{ll}
            \toprule
            \textbf{parameter} & \textbf{value} \\
            \midrule
            patch size & 2 \\
            joint depth & 8 \\
            text depth & 6 \\
            image depth & 6 \\
            dim text & 1024 \\
            dim image & 1024 \\
            dim joint attention & 1024 \\
            QK RMSnorm & true \\
            dimension per head & 64 \\
            number of heads & 8 \\
            \bottomrule
        \end{tabular}
    \end{minipage}%
    \hfill
    \begin{minipage}[t]{0.45\linewidth}%
        \centering
            \caption{Training Hyperparameters}
        \label{tab:training_hyperparams}
        \begin{tabular}[b]{llll}
            \toprule
            \textbf{Parameter} & \textbf{Stage 1} & \textbf{Stage 2} & \textbf{Stage 3} \\
            \midrule
            Num Itr & 600K & 200K & 140K \\
            EMA-$\beta$ & .99999 & .9999 & .9999 \\
            Batch Size & 256 & 512 & 512 \\
            Optimizer & AdamW & AdamW & AdamW \\
            Learning Rate & 2e-4 & 2e-4 & 2e-4 \\
            Adam-$\beta$'s & [.9, .9] & [.9, .9] & [.9, .9] \\
            Weight Decay & 0.03 & 0.03 & 0.03 \\
            \bottomrule
        \end{tabular}
    \end{minipage}%
\end{table}

\subsection{Hyperparamters}
We include the network and training hyperparameters in table \ref{tab:model_hyperparams} and \ref{tab:training_hyperparams} respectively. The total model contains $578$M parameters and the joint embedding plus a single modality is about $481$M.

%% file: tabular.tex
\section{Experimental Details on Mixed-type Tabular Data Synthesis}
\subsection{Choice of Forward Process}
We consider the same forward process discussed in \cref{app:cont_disc_forward}, with $\beta_t$ and $\sigma_s$ given as
\begin{align*}
& \beta_t = \beta_{\mathrm{start}} (1-t) + t \beta_{\mathrm{end}} \\
& \sigma_s = \frac{1-\delta}{1-(1-\delta)s} \\
& \beta_{\mathrm{start}} = 0.1, \quad \beta_{\text{end}} = 20 \quad \delta = 10^{-5}
\end{align*}

\subsection{Detailed description of datasets}
\label{app:tabular_description}
We evaluate our model on six tabular datasets (\url{https://archive.ics.uci.edu/datasets}): Adult, Default, Shoppers, Magic, Beijing, and News. Beijing and News datasets are designed for regression task while the other four datasets are for the classification task.
%%% stats table %%%
\begin{table*}[h]
\caption{Statistics for the tabular datasets.}
\label{tab:tabular_dataset_stats}
\begin{center}
\resizebox{0.7\linewidth}{!}{
\begin{tabular}{@{}c|ccccccc@{}}
\toprule
Dataset &\#Rows &\#Numerical &\#Categorical &\#Training &\#Test &Task \\ \midrule
Adult    &48,842 &6  &9  &32,561 &16,281 &Classification \\
Default  &30,000 &14 &11 &27,000 &3,000  &Classification \\ 
Shoppers &12,330 &10 &8  &11,097  &1,233  &Classification \\
Magic    &19,019 &10 &1  &17,117 &1,902  &Classification \\
Beijing  &41,757 &7  &5  &37,581 &4,176  &Regression     \\
News     &39,644 &46 &2  &35,679 &3,965  &Regression     \\ \bottomrule
\end{tabular}
}
\end{center}
\end{table*}
%%%%%%%%%%%%%%%

\subsection{Model architecture and training details}
\label{app:tabular_training}
The embedding for every numerical feature in the data is a summation of its type embedding and scale embedding. All numerical values share the same type embedding, which is a look-up table of size number of numerical features by the hidden dimension. Each numerical value is passed through a 3-layer MLP that expands a single numerical value to an embedding vector with the size of the hidden dimension. Categorical features in the data are individually embedded through a list of look-up embedding tables. The look-up embedding table has the size of the number of categories + 1 (with one extra mask token) by the hidden dimension. Then all the categorical look-up embedding tables are concatenated and treated as the categorical embedding for this dataset. The building block of our model is adopted from DiT \citep{peebles2023scalable}. The sinusoidal timestep is passed through a 2-layer MLP before input into the DiT blocks. After adding the integer positional embedding to the embedding, numerical embeddings and categorical embeddings are concatenated and input into DiT blocks. We used 4 DiT blocks with hidden dimension = 24 and number of heads = 4. The final layer splits the output into the numerical latent and a list of individual categorical latent. The latent vectors are passed into 3-layer MLPs to obtain the corresponding scores. 

The noise perturbation is the variance preserving (VP) SDE. The training loss a weighted summation of the score matching loss for numerical features and score entropy loss for categorical features. The weighting parameter is chosen to balance the numerical and discrete loss. The optimizer is AdamW with learning rate = $10^{-3}$, weight decay = 0.03, $\beta = (0.9, 0.9)$. A linear rate warm-up scheduler is used with warmup steps = 200. The training batch size is 2048. We used EMA model for final evaluation. During sampling, we use Euler method for the continuous diffusion and tau-leaping for the discrete diffusion.

\subsection{Evaluation}
\label{app:tabular_eval}

We compare our model with five most recent generative models that are specifically designed to operate on tabular data: GOOGLE \citep{liu2023goggle}, StaSy \citep{kim2022stasy}, TabDDPM \citep{kotelnikov2023tabddpm}, CoDi \citep{lee2023codi}, and T{\small AB}S{\small YN} \citep{zhang2023mixed}. GOOGLE is a VAE-based method while the other four are all diffusion-based methods. 

\begin{itemize}
    \item \textbf{Shape} is a metric computed via the Kolmogorov-Smirnov Test between continuous distributions and Total Variation Distance between the probabilities for categorical values, measures, and compares the column-wise density between real and synthetic data. 
    \item \textbf{Trend} is a metric that captures pair-wise column correlation by computing Pearson correlation for numerical columns, contingency similarity for categorical columns, and contingency similarity between bucketed numerical values and categorical values. 
    \item \textbf{MLE} is the testing accuracy of the classification or regression task on real data after training an XGBoost Classifier or an XGBoost Regressor on the synthetic tabular data. For detailed training and optimization pipelines of MLE metric, please refer to the standardized pipeline proposed in \citet{zhang2023mixed}.
    \item \textbf{$\alpha$-precision} evaluates if the synthetic data are from the same distribution as real-world data. 
    \item \textbf{$\beta$-recall} quantifies whether the synthetic data can cover the entire distribution of the real data. 
\end{itemize}

%%% Shape table %%%
\begin{table*}[!h]
\begin{center}
\caption{Performance on the \textbf{Shape} metric in percentage (\%). Higher values indicate better performance. Best performance in \textbf{bold}. Second best in \underline{underline}.}
\label{tab:shape}
\scalebox{0.85}{
\begin{tabular}{@{}cc|cccccc@{}}
\toprule

Methods & \#Parameters & Adult & Default & Shoppers & Magic & Beijing & News \\ \midrule
GOOGLE &$\sim$ 5.6M &83.03 &82.98 &77.67 &98.10 &83.07 &74.68 \\
STaSy  &$\sim$ 10.3M &88.71\scriptsize{$\pm$0.06} &94.23\scriptsize{$\pm$0.06} &90.63\scriptsize{$\pm$0.09} &93.71\scriptsize{$\pm$0.13} &93.29\scriptsize{$\pm$0.03} &93.11\scriptsize{$\pm$0.03} \\
CoDi &$\sim$ 25.0M &78.62 \scriptsize{$\pm$0.06} &84.23\scriptsize{$\pm$0.07} &68.16\scriptsize{$\pm$0.05} &88.44\scriptsize{$\pm$0.26} &83.06\scriptsize{$\pm$0.02} &67.73\scriptsize{$\pm$0.04}  \\
TabDDPM &$\sim$ 11.7M &98.25 \scriptsize{$\pm$0.03} &98.43\scriptsize{$\pm$0.08} &97.28\scriptsize{$\pm$0.13} &98.99\scriptsize{$\pm$0.09} &\underline{98.70}\scriptsize{$\pm$0.03} &21.25\scriptsize{$\pm$0.01}  \\
T{\small AB}S{\small YN}  &$\sim$ 10.7M &\underline{99.42} \scriptsize{$\pm$0.06} &\underline{99.15}\scriptsize{$\pm$0.04} &\textbf{98.57}\scriptsize{$\pm$0.24} &\textbf{99.12}\scriptsize{$\pm$0.09} &\textbf{98.88}\scriptsize{$\pm$0.05} &\textbf{98.36}\scriptsize{$\pm$0.04}  \\
T{\small AB}S{\small YN} (reproduced)  &$\sim$ 10.7M &99.29\scriptsize{$\pm$0.06} &97.12 \scriptsize{$\pm$0.09} &98.36\scriptsize{$\pm$0.10} &\underline{99.02}\scriptsize{$\pm$0.10} &96.35\scriptsize{$\pm$0.10} &\underline{98.09}\scriptsize{$\pm$0.03} \\
\midrule
Our model &$\sim$ \textbf{64K} &\textbf{99.47}\scriptsize{$\pm$0.04} &\textbf{99.36}\scriptsize{$\pm$0.09} &\underline{98.50}\scriptsize{$\pm$0.07} &98.96\scriptsize{$\pm$0.16} &97.94\scriptsize{$\pm$0.06} & 96.80\scriptsize{$\pm$0.05}
\\ \bottomrule
\end{tabular}
}
\end{center}
\vspace{-0.5cm}

\end{table*}

%%% precision table %%%
\begin{table*}[!h]
\begin{center}
\caption{Performance on the $\alpha$\textbf{-precision} metric in percentage (\%). Higher values indicate better performance. Best performance in \textbf{bold}. Second best in \underline{underline}.}
\label{tab:tabular_precision}
\scalebox{0.85}{
\begin{tabular}{@{}cc|cccccc@{}}
\toprule

Methods & \#Parameters & Adult & Default & Shoppers & Magic & Beijing & News \\ \midrule
GOOGLE &$\sim$ 5.6M  &50.68 &68.89 &86.95 &90.88 &88.81 &86.41 
\\
STaSy  &$\sim$ 10.3M &82.87\scriptsize{$\pm$0.26}  &90.48\scriptsize{$\pm$0.11} &89.65\scriptsize{$\pm$0.25} &86.56\scriptsize{$\pm$0.19} &89.16\scriptsize{$\pm$0.12} &94.76\scriptsize{$\pm$0.33}
\\
CoDi &$\sim$ 25.0M &77.58\scriptsize{$\pm$0.45} &82.38\scriptsize{$\pm$0.15} &94.95\scriptsize{$\pm$0.35} &85.01\scriptsize{$\pm$0.36} &98.13\scriptsize{$\pm$0.38} &87.15\scriptsize{$\pm$0.12}
\\
TabDDPM &$\sim$ 11.7M &96.36\scriptsize{$\pm$0.20} &97.59\scriptsize{$\pm$0.36} &88.55\scriptsize{$\pm$0.68} &98.59\scriptsize{$\pm$0.17} &97.93\scriptsize{$\pm$0.30} &0.00\scriptsize{$\pm$0.00}
\\
T{\small AB}S{\small YN}  &$\sim$ 10.7M &\textbf{99.52}\scriptsize{$\pm$0.10} &\underline{99.26}\scriptsize{$\pm$0.27} &\underline{99.16}\scriptsize{$\pm$0.22} &\textbf{99.38}\scriptsize{$\pm$0.27} &98.47\scriptsize{$\pm$0.10} &\underline{96.80}\scriptsize{$\pm$0.25}
\\
T{\small AB}S{\small YN} (reproduced)  &$\sim$ 10.7M &99.32\scriptsize{$\pm$0.22} &95.57\scriptsize{$\pm$0.33}	&\textbf{99.22}\scriptsize{$\pm$0.31}	&\underline{99.21}\scriptsize{$\pm$0.27}	&\textbf{98.87}\scriptsize{$\pm$0.15}	&96.30\scriptsize{$\pm$0.28}

\\
\midrule
Our model &$\sim$ \textbf{64K} &\underline{99.47}\scriptsize{$\pm$0.17}
&\textbf{99.47}\scriptsize{$\pm$0.21}
&98.78\scriptsize{$\pm$0.42}
&98.75\scriptsize{$\pm$0.36}
&\underline{98.49}\scriptsize{$\pm$0.24}
&\textbf{97.47}\scriptsize{$\pm$0.27}
\\ \bottomrule
\end{tabular}
}
\end{center}
\vspace{-0.5cm}

\end{table*}

%%% recall table %%%
\begin{table*}[!h]
\begin{center}
\caption{Performance on the $\beta$\textbf{-recall} metric in percentage (\%). Higher values indicate better performance. Best performance in \textbf{bold}. Second best in \underline{underline}.}
\label{tab:tabular_recall}
\scalebox{0.85}{
\begin{tabular}{@{}cc|cccccc@{}}
\toprule

Methods & \#Parameters & Adult & Default & Shoppers & Magic & Beijing & News \\ \midrule
GOOGLE &$\sim$ 5.6M  &8.80 &14.38 &9.79 &9.88 &19.87 &2.03
\\
STaSy  &$\sim$ 10.3M 
&29.21\scriptsize{$\pm$0.34} 
&39.31\scriptsize{$\pm$0.39} 
&37.24\scriptsize{$\pm$0.45} 
&\textbf{53.97}\scriptsize{$\pm$0.57} 
&54.79\scriptsize{$\pm$0.18} 
&39.42\scriptsize{$\pm$0.32}
\\
CoDi &$\sim$ 25.0M 
&9.20\scriptsize{$\pm$0.15} 
&19.94\scriptsize{$\pm$0.22} 
&20.82\scriptsize{$\pm$0.23} 
&\underline{50.56}\scriptsize{$\pm$0.31} 
&52.19\scriptsize{$\pm$0.12} 
&34.40\scriptsize{$\pm$0.30}
\\
TabDDPM &$\sim$ 11.7M 
&47.05\scriptsize{$\pm$0.25} 
&47.83\scriptsize{$\pm$0.35} 
&47.79\scriptsize{$\pm$0.25} 
&48.46\scriptsize{$\pm$0.42} 
&\underline{56.92}\scriptsize{$\pm$0.13} 
&0.00\scriptsize{$\pm$0.00}
\\

T{\small AB}S{\small YN}  &$\sim$ 10.7M 
&47.56\scriptsize{$\pm$0.22} 	
&\underline{48.00}\scriptsize{$\pm$0.35} 	
&\underline{48.95}\scriptsize{$\pm$0.28} 	 
&48.03\scriptsize{$\pm$0.23} 	
&55.84\scriptsize{$\pm$0.19} 	
&\textbf{45.04}\scriptsize{$\pm$0.34} 	
\\
T{\small AB}S{\small YN} (reproduced)  &$\sim$ 10.7M 
&\underline{47.75}\scriptsize{$\pm$0.21} 	
&42.95\scriptsize{$\pm$0.30} 	
&47.57\scriptsize{$\pm$0.44} 	
&47.92\scriptsize{$\pm$0.28} 	
&49.72\scriptsize{$\pm$0.27} 	
&44.37\scriptsize{$\pm$0.22} 
\\
\midrule
Our model &$\sim$ \textbf{64K} 
&\textbf{49.65}\scriptsize{$\pm$0.26}
&\textbf{48.29}\scriptsize{$\pm$0.32}
&\textbf{51.25}\scriptsize{$\pm$0.50}
&47.66\scriptsize{$\pm$0.38}
&\textbf{57.44}\scriptsize{$\pm$0.20}
&\underline{44.58}\scriptsize{$\pm$0.27}
\\ \bottomrule
\end{tabular}
}
\end{center}
\vspace{-0.5cm}

\end{table*}

%% file: toy_riemann.tex
\section{Additional Experiment on Riemannian-Discrete Multimodal Diffusion Model}
\label{app:riem_discrete}
In this section, we demonstrate another application of our proposed multimodal diffusion model framework by focusing on the combination of Riemannian and discrete diffusion models on the state space $\mathcal{M} \times \mathbb{X}$, where $\mathcal{M}$ is a Riemannian manifold and $\mathbb{X}$ is a finite state space. We will introduce the method and validate it on a toy example consisting of synthetic data on $\mathsf{SO}(3) \times \mathbb{X}$.
\subsection{Riemannian-Discrete Multimodal Diffusion Model}
We consider the setting where the target data distribution $p_{\mathrm{data}}(x,y)$ is defined on $\mathsf{SO}(3) \times \mathbb{X}$, where $x \in \mathsf{SO}(3)$ and $y$ is a discrete label in $\mathbb{X}$. Since $\mathsf{SO}(3)$ is a compact manifold, we choose the following as the forward process,
\begin{equation}
\label{eq:riem_disc_forward}
   \begin{cases}
    \rd X_t = \rd B^{\mathcal{M}}_t\\
    Y_s \sim \operatorname{CTMC}(Q_s) \\
    (X_0, Y_0) \sim p_{\text{data}}(x, y) 
\end{cases}
\end{equation}

\begin{figure*}[!t]
    \centering
    \includegraphics[width=0.9\textwidth]{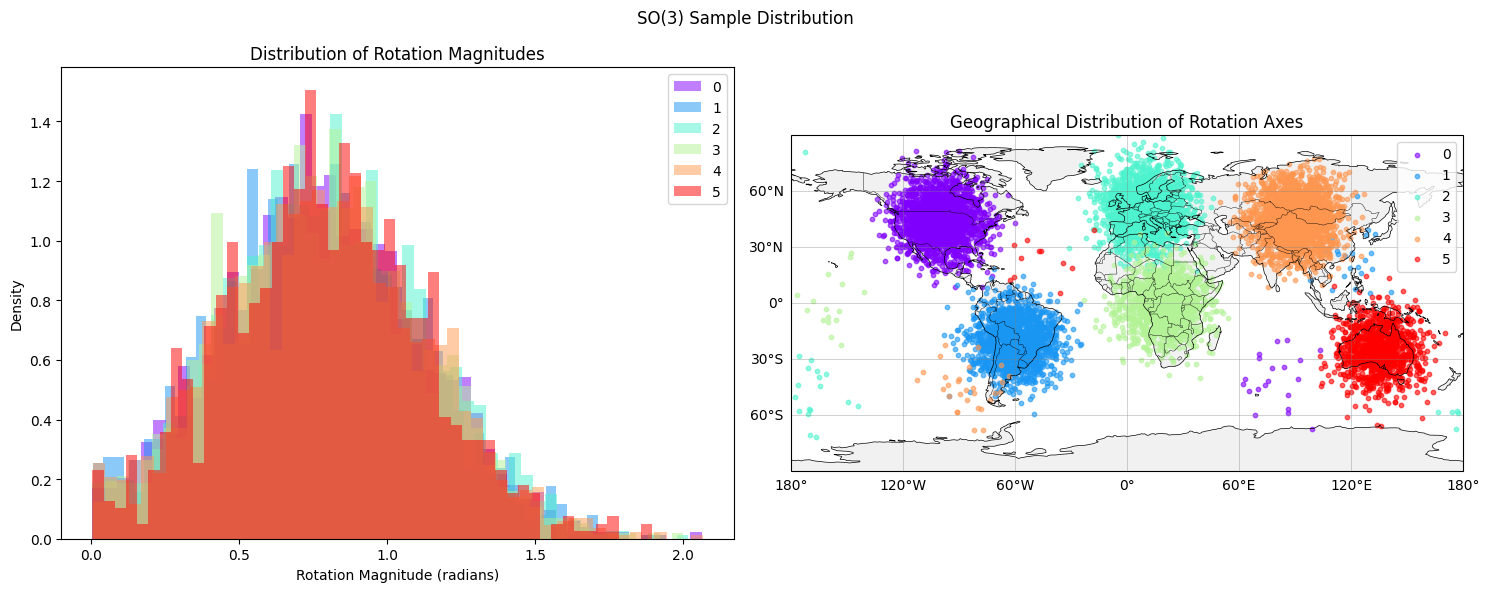}
\caption{Visual representation of ground truth labeled Riemannian data.}
\label{fig:toy_riemann}
\end{figure*}

\begin{figure*}[!t]
    \centering
    \includegraphics[width=0.9\textwidth]{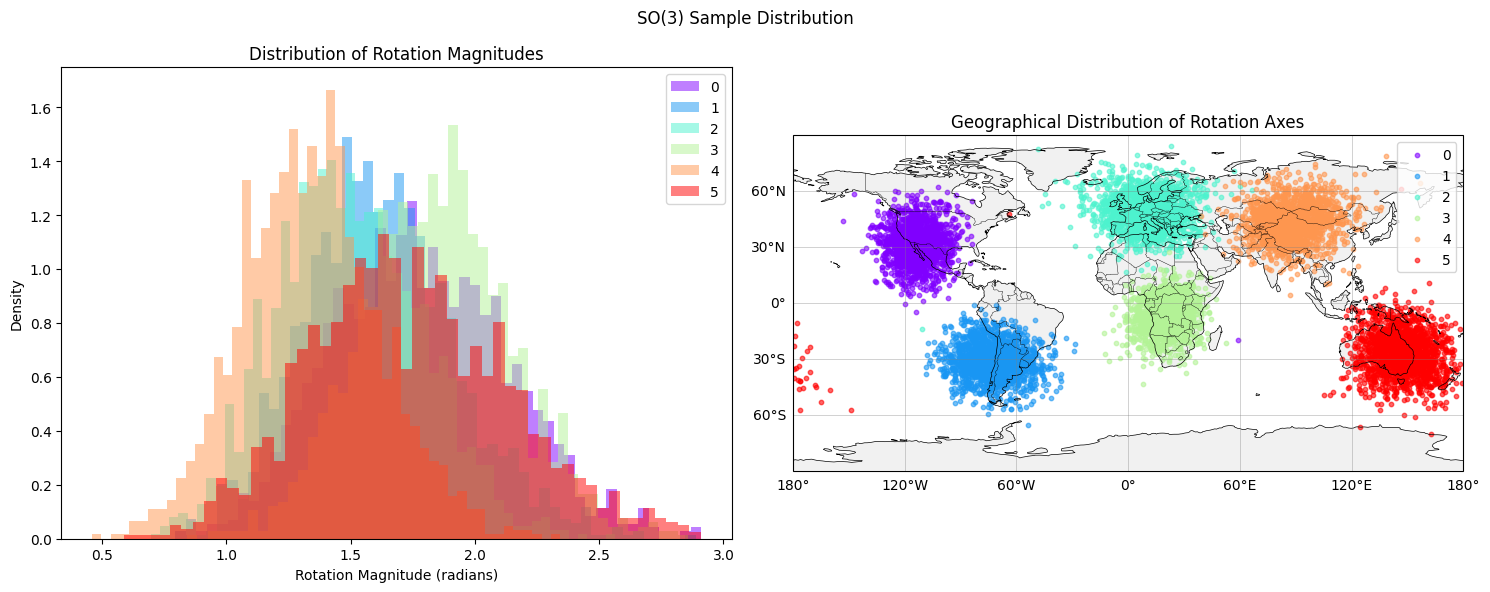}
\caption{Peformance of inferring discrete label based on $\mathsf{SO}(3)$ data.}
\label{fig:toy_riemann_ours}
\end{figure*}

\begin{figure*}[!b]
    \centering
    \includegraphics[width=0.9\textwidth]{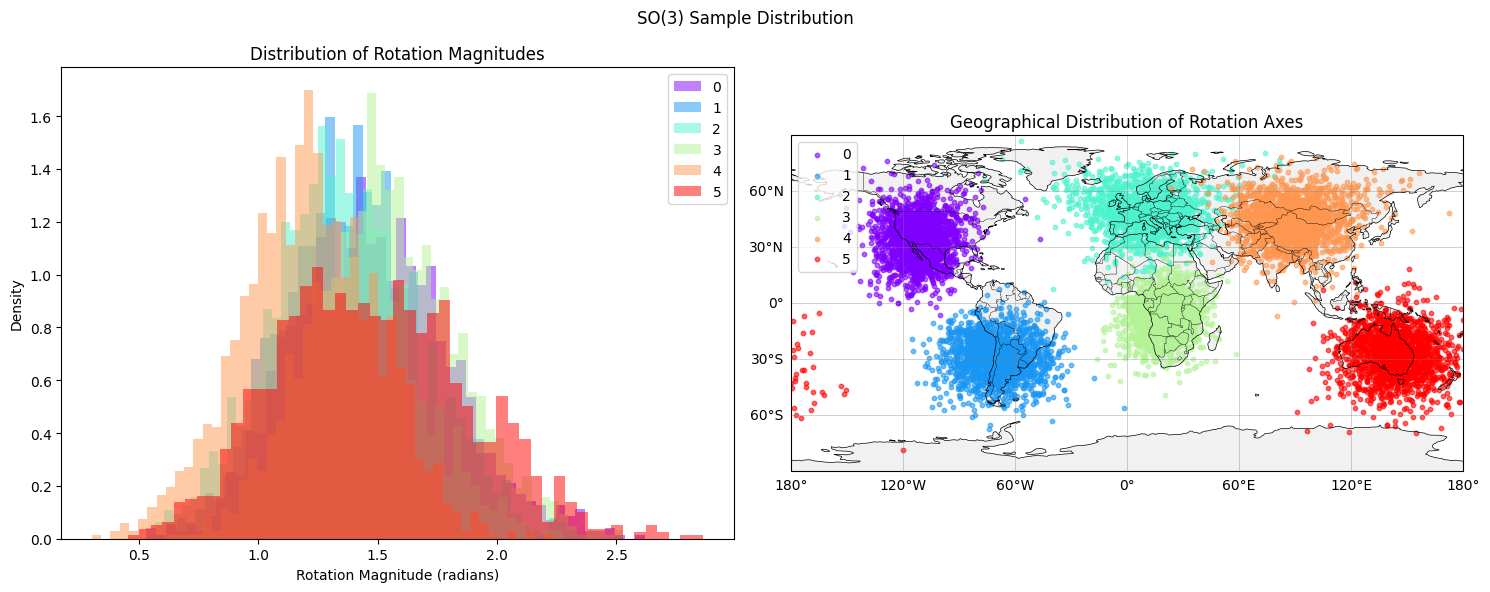}
\caption{Performance of joint generation of the labeled Riemannian data.}
\label{fig:toy_riemann_unconditional}
\end{figure*}

where $Q_s = \sigma_s Q^{\mathrm{mask}}$ is the same design choice as in \cref{eq:cont_disc_forward}, $\rd B^{\mathcal{M}}_t$ is a Brownian Motion on $\mathsf{SO}(3)$. Note that the stationary distribution of \cref{eq:riem_disc_forward} is $\operatorname{Haar}(\mathsf{SO}(3)) \times \delta_{\mathbf{M}}$, where $\operatorname{Haar}(\mathsf{SO}(3))$ is the Haar measure on $\mathsf{SO}(3)$, a generalized notion of uniform distribution. Following a similar derivation as is presented in the paper, we can derive its backward process,

\begin{equation}
\label{eq:riem_disc_backward}
   \begin{cases}
    \rd X_t =  \nabla \log p(X_t, Y_s, T-t, T-s) + \rd B^{\mathcal{M}}_t\\
    Y_s \sim \operatorname{CTMC}(\overline{Q}(X_t, t, s) \\
\end{cases}
\end{equation}
where the gradient $\nabla$ is the Riemannian gradient, and $\overline{Q}(X_t, T-t, T-s)$ is defined for $y \neq \yh$,
\begin{align*}
    \overline{Q}(x, t, s)(\yh, y) = \dfrac{p(x, \yh, T-t, T-s)}{p(x, y, T-t, T-s)} Q_{T-s}(y, \yh)
\end{align*}
\paragraph{Axis-angle parametrization.} We represent elements of $\mathsf{SO}(3)$ using the axis-angle parametrization. We introduce it briefly here. One can show that any element of $S\mathsf{SO}(3)$ can be written as $\exp(\theta K)$ where:
\begin{align*}
    K &= a\begin{pmatrix}
        0 & 0 & 0 \\
        0 & 0 & -1 \\
        0 & 1 & 0 
    \end{pmatrix} + 
    b\begin{pmatrix}
        0 & 0 & 1 \\
        0 & 0 & 0 \\
        -1 & 0 & 0 
    \end{pmatrix} + 
    c\begin{pmatrix}
        0 & -1 & 0 \\
        1 & 0 & 0 \\
        0 & 0 & 0 
    \end{pmatrix}
\end{align*} 
and $(a,b,c) \in \mathbb{S}^2$ is a vector on the sphere, $\theta \in \mathbb{R}_+$. The representation $((a,b,c), \theta)$ is called the axis-angle representation.

\paragraph{Dataset of the toy problem}
We consider a simple toy example of labeled data on $\mathsf{SO}(3)$, consisting of Gaussian mixtures, where each mode corresponds to a unique label. To create the problem, we write elements $(a,b,c) \in \mathbb{S}^2$ in spherical coordinates; in this way, only two angles need to be parameterized. We then generate a Gaussian mixture on the space of these angles. Additionally, we use a von Mises random variable for $\theta$. We present the Python code used to generate the dataset in \cref{code:dataset} and a visualization of the axis and angles in Figure \ref{fig:toy_riemann}. 

As observed in \cref{fig:toy_riemann}, we have assigned labels to different geographical locations and assigned them to distinct modes on the map. 

\paragraph{Training strategy.} We train a simple MLP using a similar strategy as the text-image model. We first train a label to $\mathsf{SO}(3)$ model and add the discrete capabilities in a second phase. To achieve this, we utilize the generalized denoising score matching loss $\cI_{\mathrm{GDSM}}$, as described in the main paper, which is derived from the generator computed based on the chosen forward process. We find that this training strategy is generally robust.

\paragraph{Results.} We present samples generated by our method using guidance $w = 4$ in Figure \ref{fig:toy_riemann_ours}, we see that our method can properly recover the data distribution. We also show the unconditional generation in Figure \ref{fig:toy_riemann_unconditional}. We demonstrate that our method and training strategy can generalize to other data modalities.

\begin{figure}[!t]
    \centering
    \includegraphics[width=\linewidth]{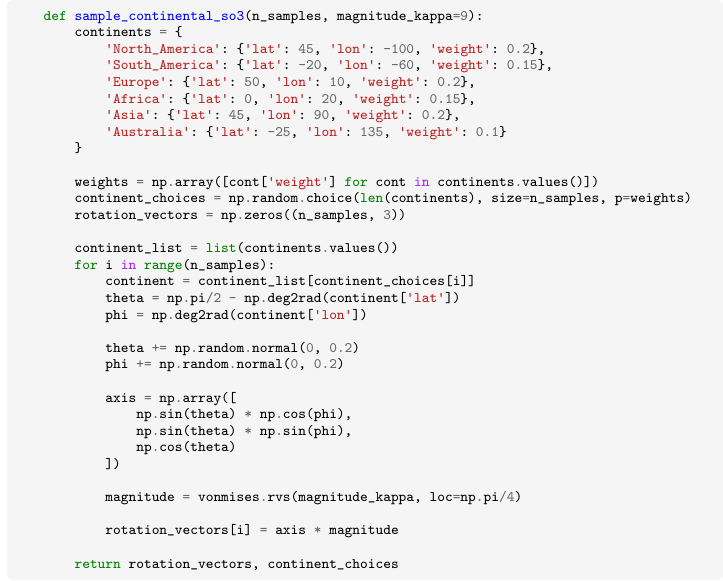}
    \caption{Code for generating the dataset}
    \label{code:dataset}
\end{figure}

% \begin{codeblock}\label{code:dataset}
% \begin{minted}[fontsize=\footnotesize, linenos, escapeinside=||]{python}
% def sample_continental_so3(n_samples, magnitude_kappa=9):
%     continents = {
%         'North_America': {'lat': 45, 'lon': -100, 'weight': 0.2},
%         'South_America': {'lat': -20, 'lon': -60, 'weight': 0.15},
%         'Europe': {'lat': 50, 'lon': 10, 'weight': 0.2},
%         'Africa': {'lat': 0, 'lon': 20, 'weight': 0.15},
%         'Asia': {'lat': 45, 'lon': 90, 'weight': 0.2},
%         'Australia': {'lat': -25, 'lon': 135, 'weight': 0.1}
%     }
    
%     weights = np.array([cont['weight'] for cont in continents.values()])
%     continent_choices = np.random.choice(len(continents), size=n_samples, p=weights)
%     rotation_vectors = np.zeros((n_samples, 3))
    
%     continent_list = list(continents.values())
%     for i in range(n_samples):
%         continent = continent_list[continent_choices[i]]
%         theta = np.pi/2 - np.deg2rad(continent['lat'])
%         phi = np.deg2rad(continent['lon'])
        
%         theta += np.random.normal(0, 0.2)
%         phi += np.random.normal(0, 0.2)
        
%         axis = np.array([
%             np.sin(theta) * np.cos(phi),
%             np.sin(theta) * np.sin(phi),
%             np.cos(theta)
%         ])
        
%         magnitude = vonmises.rvs(magnitude_kappa, loc=np.pi/4)
        
%         rotation_vectors[i] = axis * magnitude
    
%     return rotation_vectors, continent_choices
% \end{minted}
% \end{codeblock}
% \captionof{listing}{Code for generating the dataset}